\documentclass[10pt,twocolumn,letterpaper]{article}

\usepackage{iccv}
\usepackage{times}
\usepackage{epsfig}

\usepackage{graphicx}
\usepackage{amsmath,amsbsy,amsfonts,amssymb,amsthm,units,bm,mathrsfs}
\usepackage{microtype} 
\usepackage{multirow}
\usepackage{array}
\newcommand{\partitle}[1]{\smallskip \noindent \textbf{#1.}}
\usepackage{microtype}
\usepackage{enumitem}
\usepackage{algorithm, algorithmic}

\usepackage{makecell}

\usepackage[title]{appendix}
\usepackage{tabularx}
\usepackage{subfiles}

\newcolumntype{C}[1]{>{\centering\let\newline\\\arraybackslash\hspace{0pt}}m{#1}}

\newtheorem{lemma}{Lemma}[section]
\newtheorem{proposition}{Proposition}[section]
\newtheorem{theorem}{Theorem}[section]

\theoremstyle{definition}

\newtheorem{definition}{Definition}[section]

\newtheorem{remark}{Remark}

\makeatletter
\newsavebox\myboxA
\newsavebox\myboxB
\newlength\mylenA

\newcommand*\xbar[2][0.75]{%
    \sbox{\myboxA}{$\m@th#2$}%
    \setbox\myboxB\null
    \ht\myboxB=\ht\myboxA%
    \dp\myboxB=\dp\myboxA%
    \wd\myboxB=#1\wd\myboxA
    \sbox\myboxB{$\m@th\overline{\copy\myboxB}$}
    \setlength\mylenA{\the\wd\myboxA}
    \addtolength\mylenA{-\the\wd\myboxB}%
    \ifdim\wd\myboxB<\wd\myboxA%
       \rlap{\hskip 0.5\mylenA\usebox\myboxB}{\usebox\myboxA}%
    \else
        \hskip -0.5\mylenA\rlap{\usebox\myboxA}{\hskip 0.5\mylenA\usebox\myboxB}%
    \fi}
\makeatother

\newcommand{\tens}[1]{\boldsymbol{\mathscr{#1}}}
\newcommand{\vect}[1]{\ensuremath{\mathbf{#1}}}
\newcommand{\mat}[1]{\ensuremath{\mathbf{#1}}}

\newcommand{\argmax}{\mathop{\rm argmax}}

\newcommand{\tL}{\tens{L}}

\newcommand{\tS}{\tens{S}}

\newcommand{\gH}{\mathbb{H}}
\newcommand{\disG}{\mathcal{N}_{\mathbb{H}}}
\newcommand{\prob}{\mathbb{P}}

\newcommand{\la}{\langle}
\newcommand{\ra}{\rangle}

\newcommand{\X}{\mat{X}}
\newcommand{\Y}{\mat{Y}}

\renewcommand{\H}{\mat{H}}

\newcommand{\x}{\vect{x}}
\newcommand{\y}{\vect{y}}
\newcommand{\z}{\vect{z}}
\newcommand{\h}{\vect{h}}
\newcommand{\n}{\vect{n}}

\usepackage[breaklinks=true,bookmarks=false]{hyperref}

 \iccvfinalcopy 

\def\iccvPaperID{11222} 

\ificcvfinal\fi

\begin{document}

\title{Integer-arithmetic-only Certified Robustness for Quantized Neural Networks}





\author{\fontsize{10.8pt}{\baselineskip}\selectfont Haowen Lin\textsuperscript{1}, Jian Lou\textsuperscript{2,3,}\thanks{Corresponding Author.}~, Li Xiong\textsuperscript{2}, Cyrus Shahabi\textsuperscript{1}\\
\fontsize{10pt}{\baselineskip}\selectfont \textsuperscript{1}University of Southern California\ \ \ \textsuperscript{2}Emory University\ \ \ \textsuperscript{3}Xidian University\\
{\tt\small haowenli@usc.edu\ \ jlou@xidian.edu.cn\ \ lxiong@emory.edu\ \ shahabi@usc.edu}
}


\maketitle
\ificcvfinal\thispagestyle{empty}\fi

\begin{abstract}


Adversarial data examples have drawn significant attention from the
machine learning and security communities. A line of work on tackling adversarial examples is certified robustness via randomized smoothing that can provide a theoretical robustness guarantee. However, such a mechanism usually uses floating-point arithmetic for calculations in inference and requires large memory footprints and daunting computational costs. These defensive models cannot run efficiently on edge devices nor be deployed on integer-only logical units such as Turing Tensor Cores or integer-only ARM processors. To overcome these challenges, we propose an integer randomized smoothing approach with quantization to convert any classifier into a new smoothed classifier, which uses integer-only arithmetic for certified robustness against adversarial perturbations. We prove a tight robustness guarantee under $\ell_2$-norm for the proposed approach. We show our approach can obtain a comparable accuracy and $4 \times \sim 5 \times$ speedup over floating-point arithmetic certified robust methods on general-purpose CPUs and mobile devices on two distinct datasets (CIFAR-10 and Caltech-101).

\end{abstract}


\section{Introduction}

Recent works in deep learning have demonstrated that well-trained deep neural networks can easily make wrong predictions with high confidence when a sample is perturbed with a small but adversarially-chosen noise \cite{szegedy2013intriguing,goodfellow2014explaining,papernot2016limitations}. To defend against
these attacks, several works have proposed to develop defensive
techniques and improve the
robustness of deep neural networks\cite{agarwal2019improving,luo2018towards,li2018robust,wang2021certified}. A recent promising line of work focuses on developing certifiably robust classifiers that promise no adversarially perturbed examples within a certified region can alter the classification result \cite{cohen2019certified,weng2018towards,lecuyer2019certified}. 
Such certified defenses provide a rigorous guarantee against norm-bounded perturbation attacks and, more importantly, ensure effectiveness under future stronger attacks \cite{cohen2019certified,li2018certified,lecuyer2019certified,tjeng2017evaluating}. One primary theoretical tool for providing the robustness guarantee is randomized smoothing, which derives a smoothed classifier from the base classifier via injecting designated noises, e.g., Gaussian noise. Multiple repeated inferences through the base classifier (i.e., Monte-Carlo estimation) are required to approximate the smoothed classification result for robustly predicting or certifying a single example.

Despite the promising results achieved by many certified robustness algorithms, existing methods almost exclusively focus on floating-point (FP) represented neural networks. However, 
the vastly adopted compressed neural network models are considered indispensable when one wishes to deploy the networks on storage-, computing resources- and power consumption-limited platforms such as edge devices, mobile devices, and embedded systems. In practice, one of the most successful and mainstream compression methods is quantization \cite{jacob2018quantization,dong2019hawq,wu2018mixed,zhang2018lq}. Quantization is a simple yet effective technique that compresses deep neural networks into smaller sizes by replacing model weights and activations from 32-bit floating-point (FP32) with low-bit precision, e.g., 8-bit integer (int8) \cite{zhang2018lq}. Both storage and computational complexity can be reduced using low-bit quantized neural networks \cite{jacob2018quantization}. Moreover, Jacob et al. \cite{jacob2018quantization} have proposed an integer-arithmetic-only quantization framework that further accelerates inference by using integer multiplication and accumulation for calculation. Performing inference using integer-arithmetic-only operations has several advantages in real application scenarios. For example, it resolves the limitation that floating-point networks cannot be deployed onto digital computing devices such as the recent Turing Tensor Cores or traditional integer-only ARM processors. Moreover, computing with integer arithmetic significantly reduces computing power, making them attractive for energy-constrained edge deployment and some cost-sensitive cloud data centers \cite{faraone2019addnet}.

\begin{figure}[t]
    \centering
    \includegraphics[width=0.55\linewidth]{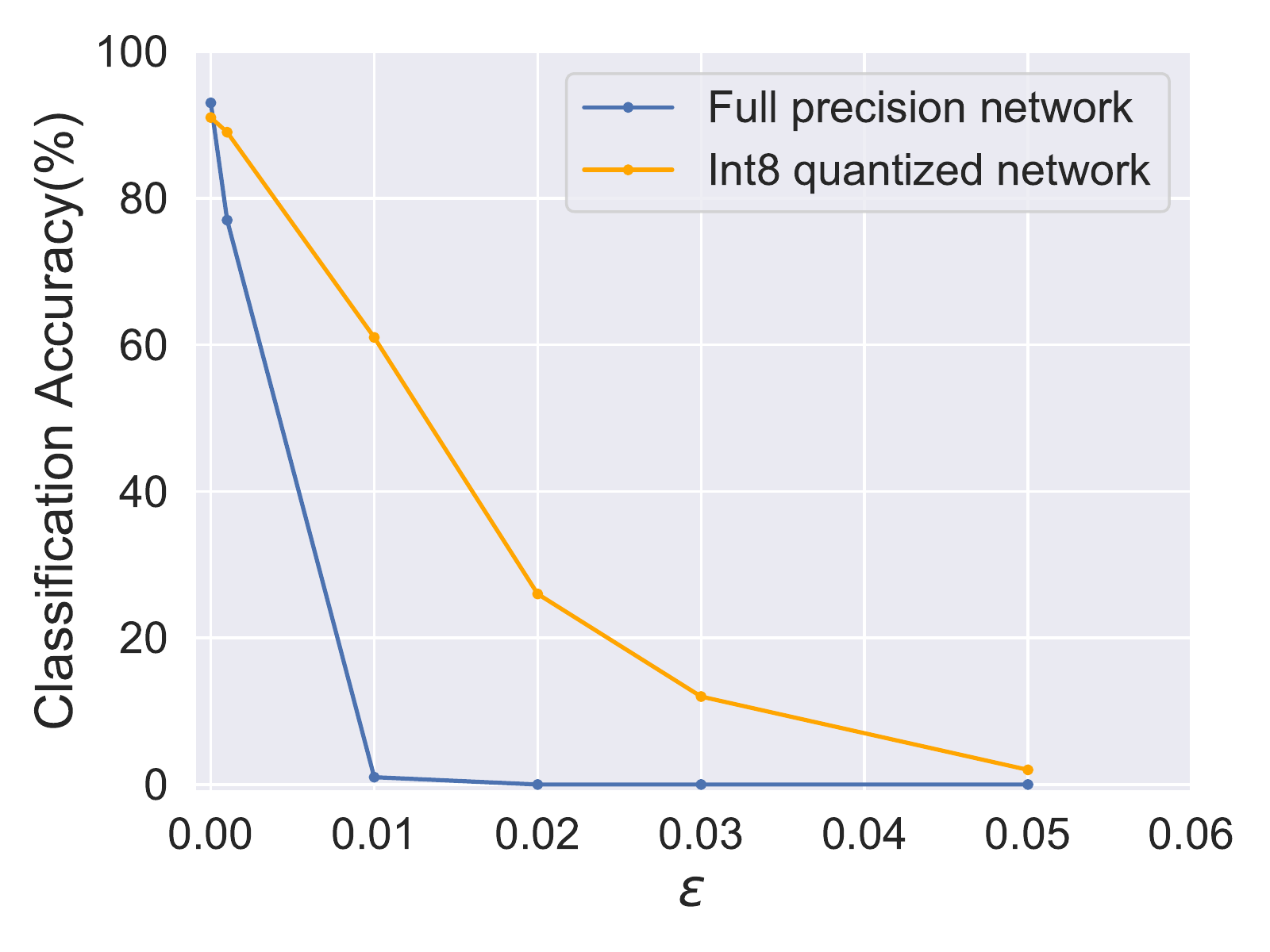}
    \caption{A demonstration of the adversarial perturbation attacks on undefended full precision network (blue) and 8-bit integer network (yellow). The x-axis represents the radius of the noise projected into clean images under $\ell_\infty$ ball (parameterized by $\varepsilon$). Given 100 clean images, the full precision and quantized network achieved 93\% and 91\% accuracy, respectively. When images are attacked by Projected Gradient Descent attack targeting the full precision model, the accuracy of both classifiers begins to drop. Details are deferred to Supplement.}
    \label{fig:attack}
\end{figure}

Given the under-studied situation of the certified robustness for quantized neural networks, the following research questions naturally arise: Q1. Are adversarial perturbations still effective on quantized neural networks? Q2. Can we reuse the current certified robustness defenses on quantized neural networks? Q3. If not, how can we design a certifiably robust defense making full considerations of the characteristics of quantized neural networks? Q1 and Q2 can be readily answered as follows, \\
\textbf{For Q1}, we consider the following demonstrating example. We generate adversarial perturbations using  Projected Gradient Descent attack \cite{madry2017towards} and inject the perturbations to 100 randomly selected clean images from CIFAR-10. The result is presented in Figure \ref{fig:attack}. Although the quantized network manifests slightly stronger robustness than the full precision model, adversarial perturbations can still sabotage the classification performance of the quantized neural network even when the perturbation noise is small. More severely, since the adversarial perturbation used in the example does not consider the characteristics of the quantized neural network as a priori, stronger attacks can be devised once such information is exploited. Thus, it is pressing for us to study certified robustness for quantized neural networks.\\
\textbf{For Q2}, a new certified robustness mechanism tailored to quantized neural networks is indeed necessary. The reason is that existing certified robustness methods rely on floating-point operations, which are incompatible with integer-arithmetic-only devices. Sometimes even when deployed on platforms that do support floating-point operations, a new integer-arithmetic-only certified robustness method can still be desirable due to the efficiency reason, especially considering that most certified robustness methods invoke repeated inferences to certify on a single example and incur large inference time.

\begin{figure}[t]
    \centering
    \includegraphics[width=0.95\linewidth]{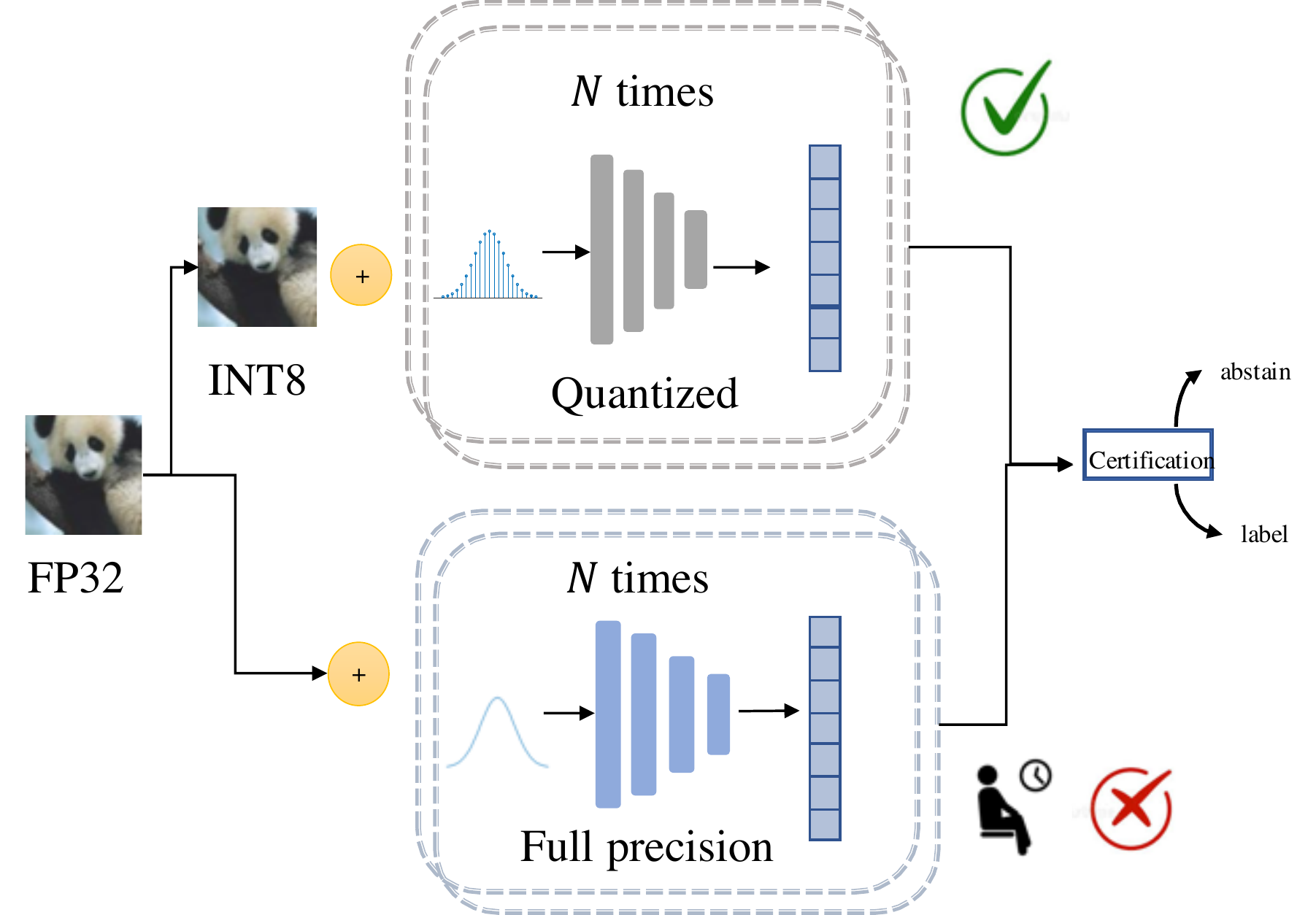}
    \caption{A comparison of certification procedure between IntRS and the floating-point network. Certification requires repeating inference N times. Compared with original floating-point network, IntRS achieves comparable certified accuracy with less inference time.}
    \label{fig:framework}
\end{figure}



As a result, our main effort in this paper is to answer \textbf{Q3} with a novel and first integer-arithmetic-only certified robustness mechanism for quantized neural networks against adversarial perturbation attacks. An illustration of our framework and comparison with the existing certified robustness defenses is in Figure \ref{fig:framework}. In summary, we make the following contributions:
\begin{itemize}
    \item We devise a new integer-arithmetic-only randomized smoothing mechanism (abbreviated as IntRS) by incorporating the discrete Gaussian noise. More importantly, we rigorously calibrate the certified robustness behavior of IntRS in theory.
    In practice, we introduce the quantization- and discrete data augmentation- aware as well as the common Monte-Carlo-based estimation.
    \item We perform experiments with 1) two different base neural network architectures with medium and large scales; 2) two different datasets; 3) two different types of computing devices (general-purpose computer and mobile device), which verify that IntRS achieves similar robustness and accuracy performance, and $4\times$ to $5\times$ efficiency improvement in the inference stage over existing floating-point randomized smoothing method for the original full precision neural networks. 
\end{itemize}

\section{Related Work}

\partitle{Certified Defenses} 
A certifiably robust classifier guarantees a stable classification for any input $\x$ within a certain range (e.g., within an $L_p$-norm ball) \cite{wong2018provable}. Certification methods usually fall into two categories: exact and conservative (relaxed). Given a base classifier $f$ and an input $\x$, exact certification techniques answer the question of whether or not there exists an adversarial perturbation for $f$ at $\x$. On the other hand, if the conservative algorithms make a certification, it guarantees that there is no perturbation exist, but it could refuse to make certification even for a clean data point. Exact methods, usually based on Mixed Integer Linear Programming \cite{tjeng2017evaluating,singh2019boosting,fischetti2017deep}, and Satisfiability Modulo Theories \cite{kurakin2016adversarial,athalye2018robustness,katz2017reluplex} can be intractable to compute and difficult to scale up to large and even moderate size neural networks. Conservative certification can be more scalable than exact methods, but it is still challenging to obtain a robust network in the large-scale setting and apply it to high-dimensional data.

\partitle{Randomized Smoothing}
To further scale up certified defenses on larger networks encountered in practice, randomized smoothing methods have
been proposed, which provide certified robust guarantee in theory and Monte-Carlo estimation-based certification and prediction in practice. Lecuyer et al. (PixelDP) first applied differential privacy to prove robustness guarantees for randomized smoothing classifier in the case of $\ell_0$- norm attack \cite{lecuyer2019certified}. Cohen et al. \cite{cohen2019certified} used Neyman-Pearson theorem to prove tight certified robustness bound through the smoothed classifier in $\ell_2$-norm certified defense. More recently, Dvijotham et al. \cite{dvijotham2020framework} extended randomized smoothing to handle arbitrary $\ell_p$ settings and prove the robustness using f-divergence. However, all these certified robustness analyses have relied on the continuous probability distribution, which incurs compatibility issues with the integer-arithmetic-only platforms. In addition, Monte-Carlo estimation requires a large number of repeated inferences during certification, which leads to huge computation time, especially for large neural networks. Finally, we note that there are existing works that propose randomized smoothing variants for discrete input domain \cite{DBLP:conf/nips/LeeYCJ19,bojchevski2020efficient,ye2020safer}, which are somewhat closer to our problem. However, they merely consider the discrete nature of the input space while still pay little attention to neither applicability to the quantized neural networks nor to the integer-arithmetic-only restriction of many applications with quantized neural networks.

\partitle{Model Compression: Quantization}
Quantization aims to compress a network and save computation by changing the weights and activation of a neural network from 32-bit floating-point representation into lower bit-depth representations. The weights can be quantized to 16-bit \cite{gupta2015deep,das2018mixed}, 8-bit \cite{jacob2018quantization,wang2020apq}, 4-bit \cite{wang2020towards,zhou2017balanced} or even with 1-bit (also known as binary quantization) \cite{hubara2016binarized,zhou2016dorefa}. Existing techniques can be roughly divided into two categories: quantization during training and quantization for inference. Early quantization during training is applied to reduce the network size and make the training process more computational efficient \cite{choudhary2020comprehensive}. Guo et al. \cite{guo2016dynamic} applied vector quantization to compress the weighing matrix. On the contrary, quantization for inference mainly focuses on accelerating inference speed. For example, Han et al. \cite{han2015deep} proposed a three-stage pipelin: pruning, quantization, and fine-tuning to determine the bit widths for convolution layers \cite{han2015deep}. A more recent quantization scheme has been proposed in \cite{jacob2018quantization}, which inserts ``fake'' quantization operations during training using floating-point operations and builds quantized parameters during the inference stage. It simulates and decreases quantization error while performing efficient inference on integer-arithmetic-only hardware. We will further extend this quantization scheme and incorporate it into our certified robust networks.

\partitle{Quantization with Robust Training} 
Several works try to address the challenge of training deep neural networks on security-critical and resource-limited applications. Gui et al. \cite{gui2019model} construct constrained optimization formulation to integrate various existing compression techniques (e.g., pruning, factorization, quantization) into adversarially robust models. However, their approach is only applied to defend specific adversarial attacks and does not provide certified robustness guarantees. Another related line of research aims to connect model
compression with the robustness goal through pruning \cite{sehwag2020hydra,ye2019adversarial, sehwag2019towards}. However, most of the works focus on empirical adversarial defenses, and none of them support deployment with integer-arithmetic-only platforms. 


\section{Proposed Method}

In this section, we present a new certified robustness defense against adversarial perturbation attack under the integer-arithmetic-only context.

\subsection{Preliminaries}
First, we formally define the adversarial perturbation attack and certified robustness under the integer-arithmetic-only context, as follows.
\begin{definition} ($\ell_2$-Norm Bounded Integer Adversarial Perturbation Attack)
For any input $\x\in \gH^d$, where $\gH$ is a discrete additive subgroup, the $\ell_2$-norm bounded integer adversarial perturbation attack with magnitude $L$ perturbs $\x$ to $\x+\bm{\delta}$ in order to alter the classification result, where $\bm{\delta}\in\gH^d$ and $\|\bm{\delta}\|_2^2 \leq L^2$. Denote all possible $\x+\bm{\delta}$ within $L$ distance by $\mathbb{B}_{\gH}(\x,L)$.
\end{definition}

\begin{definition} (Certified Robustness to $\ell _2$-Norm Bounded Integer Adversarial Example Attack with magnitude $L$)
Denote a multiclass classification model by $f(\x): \mathcal{X} \mapsto c\in \mathcal{C}$, where $c$ is a label in the labels set $\mathcal{C} = \{1,...,C\}$. In general, $f(\x)$ outputs a vector of scores $f^{\y}(\x) = (f^{y_1},...,f^{y_C})\in\mathcal{Y}$, where $\mathcal{Y}=\{\y: \sum _{i=1}^C f^{y_i} = 1, f^{y_i}\in [0,1]\}$, and $c=\argmax_{i\in\mathcal{C}} f^{y_i}$. A predictive model $f(\x)$ is robust to $\ell _2$-norm integer adversarial perturbation attack with magnitude $L$ on input $\x$, if for all $\x'\in \mathbb{B}_{\H}(\x,L)$, it has $f(\x)=f(\x')$, which is equivalent to 
\begin{equation}
    f^{y_{c}}(\x') > \max_{i\in \mathcal{C}: i\neq c} f^{y_i}(\x').
\end{equation}
\end{definition}

The following lemma is utilized in the proof of the certified robustness guarantee.
\begin{lemma} (Neyman-Pearson Lemma \cite{neyman1933ix})
\label{lemma.discrete.Neyman.Pearson}
Let $X, Y$ be random variables in $\gH^d$ with probability mass function $p_X, p_Y$. Let $h:\gH^d \mapsto \{0,1\}$ be a random or deterministic function. It has:\\
1) If $\tL(\z)=\{\z\in\gH^d: \frac{p_Y(\z)}{p_X(\z)}\leq \alpha\}$ for some $\alpha >0$ and $\prob(h(X)=1) \geq \prob(X\in \tL(\z))$, then $\prob(h(Y)=1)\geq \prob(Y\in\tL(\z))$;\\
2) If $\tL(\z)=\{\z\in\gH^d: \frac{p_Y(\z)}{p_X(\z)}\geq \alpha\}$ for some $\alpha >0$ and $\prob(h(X)=1) \leq \prob(X\in \tL(\z))$, then $\prob(h(Y)=1)\leq \prob(Y\in\tL(\z))$.
\end{lemma}

\subsection{Integer Randomized Smoothing for Quantized Neural Networks}

We propose the integer randomized smoothing (abbreviated as IntRS) for quantized neural networks to ensure certified robustness against $L$-$\ell _2$-norm bounded integer adversarial perturbation, which involves only integer arithmetic operations. For this purpose, we utilize the following \emph{discrete} Gaussian random noise.
\begin{definition} (Discrete Gaussian Distribution) The discrete Gaussian distribution $\disG(\mu,\sigma ^2)$ is a probability distribution supported on $\gH$ with location $\mu$ and scale $\sigma$. Its probability mass function is defined as follows,
\begin{equation}
    \label{eq.disGaussian.mass}
    \prob_{X\sim \disG(\mu,\sigma ^2)}[X=x] = \frac{e^{-(x-\mu)^2/2\sigma ^2}}{\sum _{h\in \gH}e^{-(h-\mu)^2/2\sigma ^2}}.
\end{equation}
\end{definition}
The discrete Gaussian noise-based integer randomized smoothing (IntRS) is defined as follows. 
\begin{definition} (Integer Randomized Smoothing with Discrete Gaussian Noise)
For an arbitrary base quantized classifier defined on $\gH^{d}$, for any input $\x \in \gH^{d}$, the smoothed classifier $g(\x)$ is defined as 
\begin{equation}
\label{eq.discrete.gaussian}
    g(\x) = \argmax_{c\in\mathcal{C}} \prob(f(\x+\bm{\delta})=c),~ \bm{\delta}\sim \disG(0,\sigma ^2 \bm{I}_{d}).
\end{equation}
\end{definition}
\begin{remark}
One may wonder why we cannot re-utilize the existing randomized smoothing that relied on continuous Gaussian noise (e.g., \cite{cohen2019certified}) by rounding the noise for quantized neural networks. There are two reasons: 1) Sampling and injecting continuous Gaussian noise still requires float-number operations, which raises the compatibility issue; 2) Although 1) can be addressed via pre-storing rounded continuous Gaussian noises, the rounding error introduces randomness that is difficult to precisely calibrate, which contradicts our goal to guarantee certified robustness rigorously. To form a theoretical proof, we utilize the discrete Gaussian noise, whose distribution is precisely described by eq.(\ref{eq.discrete.gaussian}).
\end{remark}
Although the IntRS mechanism parallels the randomized smoothing in the full precision context, it is nontrivial to establish its certified robustness guarantee. A simpler yet \emph{indirect} way is to first regard IntRS as a differential privacy mechanism \cite{canonne2020discrete} and then establish the certified robustness in a similar manner as PixelDP \cite{lecuyer2019certified}. We take another \emph{direct} way, which provides a tighter certified robustness guarantee. The key to our certified robustness guarantee of the IntRS mechanism is by extending the Neyman-Pearson lemma to the discrete Gaussian distribution, as summarized by the following Proposition.
\begin{proposition} (Neyman-Pearson for Discrete Gaussian with Different Means)
    For $\X\sim \disG(\x,\sigma ^2\bm{I}_{d})$, $\Y\sim \disG(\x+\bm{\delta},\sigma ^2\bm{I}_{d})$, and $h:\gH^d \mapsto \{0,1\}$, it has:\\ 1) If $\tL(\z) = \{\z\in\gH^d: \la \z,\bm{\delta} \ra \leq \sigma^2\ln \alpha + \frac{1}{2}(\|\bm{\delta}\|_2^2+2\la \x,\bm{\delta} \ra)\}$ for some $\alpha$ and $\prob(h(\x)=1) \geq \prob(\x\in\tL(\z))$, then $\prob(h(\y)=1) \geq \prob(\y\in\tL(\z))$;\\ 2) If $\tL(\z) = \{\z\in\gH^d: \la \z,\bm{\delta} \ra \geq \sigma^2\ln \alpha + \frac{1}{2}(\|\bm{\delta}\|_2^2+2\la \x,\bm{\delta} \ra)\}$ for some $\alpha$ and $\prob(h(\x)=1) \leq \prob(\x\in\tL(\z))$, then $\prob(h(\y)=1) \leq \prob(\y\in\tL(\z))$.
\end{proposition}
\begin{proof}
    Proof can be found in Appendix H.1.
\end{proof}

The following Theorem provides the certified robustness to the $\ell_2$-norm bounded integer adversarial perturbation achieved by IntRS.
\begin{theorem} (Certified Robustness via Integer Randomized Smoothing) 
\label{lemma.np.discrete.gaussian}
Let $f: \mathbb{R}^d \mapsto \mathcal{Y}$ be the base classifier. Let $\n\sim\disG(0,\sigma^2 \bm{I}_d)$. Denote the randomized smoothed classifier by $g(\x) = \arg\max_{c\in \mathcal{Y}} \prob[f(\x+\n)=c]$. If there exists $c_A \in \mathcal{Y}$ such that the following relation stands
\begin{equation}
    \prob[f(\x+\n)=c_A] \geq p^{lb}_{c_A} \geq p^{ub}_{c_B} \geq \max_{c_B:c\neq c_A} \prob[f(\x+\n)=c],
\end{equation}
then the randomized smoothed classifier $g(\x)$ is certified robust to perturbation $\|\bm{\delta}\|_2^2 \leq R^2$, i.e., $g(x+\bm{\delta})=c_A$, where $R$ is the certified radius.
\label{thm:Certified robustness via Integer Gaus}
\end{theorem}

\begin{proof}
    Let $\X = \x+\n \sim \disG(\x,\sigma^2\mathbb{I})$, $\Y = \x+\n + \bm{\delta} \sim \disG(\x+\bm{\delta},\sigma^2\bm{I}_d)$. To show that $g(\x+\bm{\delta})=c_A$, we first prove 
    \begin{equation}
    \begin{split}
        & \prob[f(\x+\n+\bm{\delta})=c_A] > \prob[f(\x+\n+\bm{\delta})=c_B]\\
        & \Longleftrightarrow \prob[f(\Y)=c_A] > \prob[f(\Y)=c_B],
    \end{split}
    \end{equation}
    given that $\prob[f(\X)=c_A] \geq p^{lb}_{c_A} \geq p^{ub}_{c_B} \geq \prob[f(\X)=c_B]$. Notice that that $p^{lb}_{c_A} = \prob[X\in \tS_A]$, where $\tS_A = \{\z:\la\z-\x,\bm{\delta}\ra \leq \sigma \|\bm{\delta}\|_2 \Phi^{-1}_{\disG}(p^{lb}_{C_A})\}$ and $p^{ub}_{c_B} = \prob[X\in \tS_B]$, where $\tS_B = \{\z:\la\z-\x,\bm{\delta}\ra \geq \sigma \|\bm{\delta}\|_2 \Phi^{-1}_{\disG}(1-p^{lb}_{C_B})\}$. Then, by Lemma \ref{lemma.np.discrete.gaussian}, we have
    \begin{gather}
        \prob[f(\Y)=c_A] \geq \prob[\Y\in \tS_A]  = p^{lb}_{c_A}\\
        \prob[f(\Y)=c_B] \leq \prob[\Y\in \tS_B]  = p^{ub}_{c_B},
    \end{gather}
    if $\delta$ satisfies the constraint:
    \begin{equation}
        \|\delta\|_2^2 < (\frac{\sigma}{2}(\Phi^{-1}_{\disG}(p_{c_A}^{lb}) - \Phi^{-1}_{\disG}(p_{c_B}^{ub})))^2.
    \end{equation}
    Thus, we have proved $\prob[f(\Y)=c_A] 
    \geq p^{lb}_{c_A} \geq p^{ub}_{c_B} \geq \prob[f(\Y)=c_B]$.
\end{proof}

\partitle{Practical IntRS} As with all randomized smoothing methods, we resort to Monte-Carlo estimation to approximate the randomized smoothing function $g(\x)$, because it is challenging to obtain the precise $g(\x)$. Essentially, for certification on example $\x$, we repeatedly evaluate $f(\x+\n)$, $n\sim\disG(0,\sigma^2 \bm{I}_d)$, and obtain robust prediction or certification result based on the aggregated evaluations. Algorithm \ref{alg.certification.prediction} summarizes the robust certification, which follows \cite{cohen2019certified}.

\begin{algorithm}[h]
\caption{Monte-Carlo estimation and aggregated evaluation for certified robust certification}
\label{alg.certification.prediction}
\begin{algorithmic}[1]
\REQUIRE Base function $f(\cdot)$, inference sample $\x$, Gaussian noise std $\sigma$, repeated number $N_1$, $N_2$, and confidence level $\alpha$. \\

\textbf{Certification:}
\STATE Repeat $N_1$ inferences on $f(\x+\n)$, where $\n\sim \disG(0,\sigma^2 \bm{I}_d)$.
\STATE Collect prediction results: $(n_1,\hat{c}_A):$ highest prediction count and its label;
\STATE Repeat $N_2$ inferences on $f(\x+\n)$, where $\n\sim \disG(0,\sigma^2 \bm{I}_d)$.
\STATE Collect prediction results: $n_2:$ count of $\hat{c}_A$;
\STATE Let $\underline{p_A}$ be the one-sided $(1-\alpha)$ lower confidence interval for ${\tt Binomial}(N_2,n_2)$;
\IF{$\underline{p_A}>\frac{1}{2}$}
\STATE \textbf{Return} $\hat{c}_A$ and radius $\sigma\Phi^{-1}(\underline{p_A})$;
\ELSE \STATE ABSTAIN
\ENDIF
\end{algorithmic}
\label{algoritm:certification}
\end{algorithm}

The following Theorem and its proof summarize that the practical prediction and certification of IntRS can be carried out via integer-arithmetic-only operations.
\begin{theorem} 
    The prediction and certification procedure of the practical IntRS can be executed with integer-arithmetic-only operations.
\end{theorem}

\begin{proof}
    The overall prediction and certification procedure consist of the following integer-arithmetic only steps:
    \begin{enumerate}
        \item Repeat the inference $N,N_1,N_2$ times to estimate $\hat{c}_A$, $\hat{c}_B$ and counts, which evaluates $\x+\bm{\delta}$ with quantized neural network. 
        \item Compute $\Phi^{-1}_{\disG}(\underline{p_{A}})$. Since the distribution is discrete, the cumulative density function has exactly $|\gH|$ elements, which means the cumulative density function can be precomputed and its inverse can be obtained by looking up during inference.
        \item Return the squared $\ell _2$ norm. Since we do not take least square norm, the squared $\ell _2$ norm is closed within $\gH$. 
    \end{enumerate}
    Thus, the robust prediction and certification of the IntRS mechanism can be implemented with integer-arithmetic-only operations.
\end{proof}

\subsection{Practical Training: Discrete Noise and Quantization Aware Training}

In this section, we present how to incorporate Quantization Aware Training (QAT) with our randomized mechanism. 

Quantization replaces floating-point number representations with low precision fixed-point representations. Since the quantization scheme replaces floating-point weights/activations by N (N=256 for 8-bits) fixed-point numbers, if we directly quantize the pre-trained network for inference, rounding errors accumulate, leading to a significant drop in the classification performance. A solution to rectify these drifting errors is to quantize the network during training. However, representing parameters by integers becomes a challenge in backward pass since direct optimization over a discrete domain generates null gradients. To address this issue, QAT keeps carrying out computation using 32-bits floating-point arithmetic while simulating quantization error. In particular, during the forward pass, weight is quantized into integers and then convoluted with the input, which is represented with floating numbers. On the other hand, backpropagation is the same as that of the full precision model.

In order for the base classifier $f$ to classify the labeled examples correctly and robustly, $f$ needs to classify them with their true labels consistently. Thus, we train $f$ with noisy images, which is also adopted by almost all certified robustness approaches. To combine with the QAT scheme, in practice, we train the model with discrete Gaussian data augmentation at variance $\sigma$, which is the same as used for prediction and certification. For each training data point, we randomly draw noise from discrete Gaussian distribution. We follow the sampling strategy in \cite{canonne2020discrete}. Note that although the algorithm for sampling from discrete Gaussian distribution runs $O(1)$ on average, since we need to add noise onto every channel of every pixel in the image, sampling noise for a batch of $32 \times 32$ images takes more than 3 mins. To accelerate training speed, we first sample noises for a batch and store these noises in advance. During training, we randomly shuffle and add the noise to the input.


\section{Experiment}

 \begin{figure*}
    \centering
    \includegraphics[width=1\linewidth]{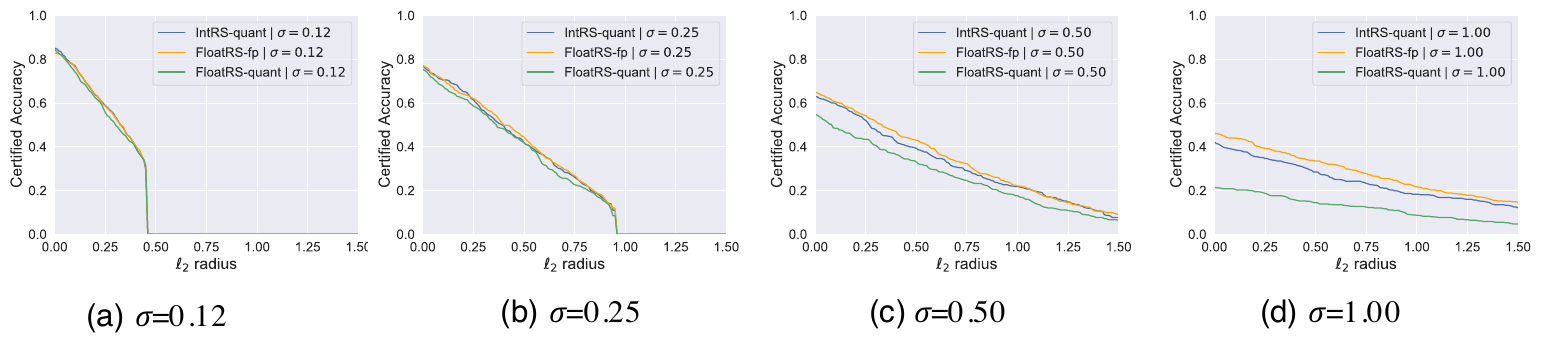}
    \caption{Certified accuracy of IntRS-quant trained CIFAR-10 classifiers vs Float-fp and FloatRS-quant. (Blue,Orange,Green) per $\sigma$. }
    \label{fig:CA_CIFAR}
\end{figure*}

\begin{figure*}
    \centering
    \includegraphics[width=1\linewidth]{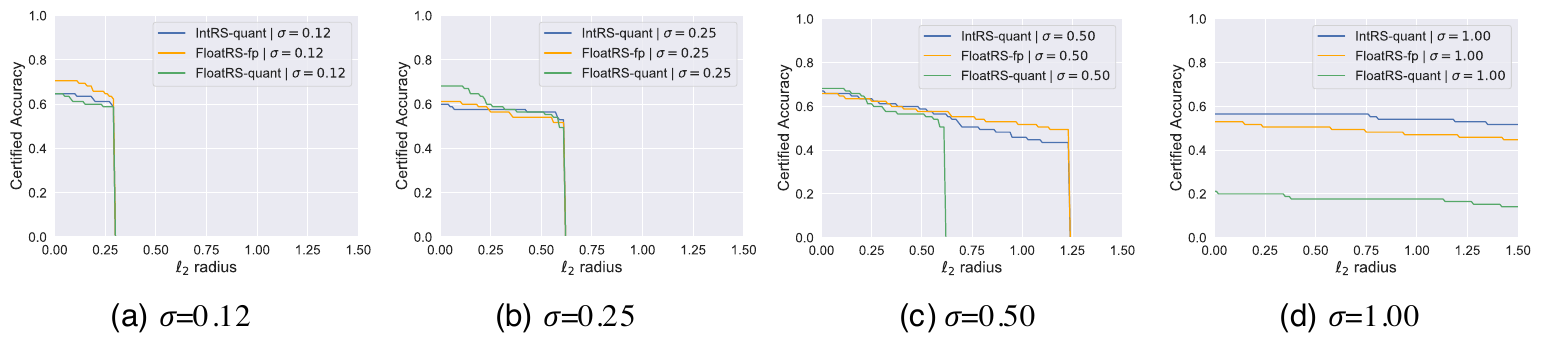}
    \caption{Certified accuracy of IntRS-quant trained Caltech-101 classifiers vs Float-fp and FloatRS-quant. (Blue,Orange,Green) per $\sigma$. }
    \label{fig:CA_Caltech}
\end{figure*}

 \begin{figure}
    \centering
    \includegraphics[scale=1.15]{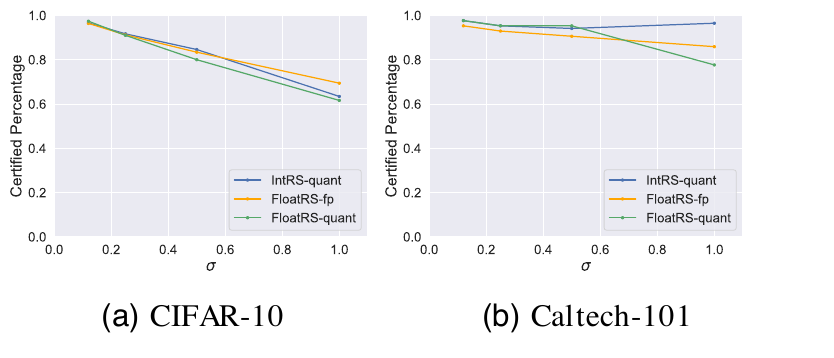}
    \caption{Certified Percentage of IntRS-quant trained Caltech-101 classifiers vs Float-fp and FloatRS-quant per $\sigma$ over CIFAR-10 (Left) and Caltech-101 (Right).}
    \label{fig:CR}
\end{figure}

\begin{figure*}
    \centering
    \includegraphics[scale=1.15]{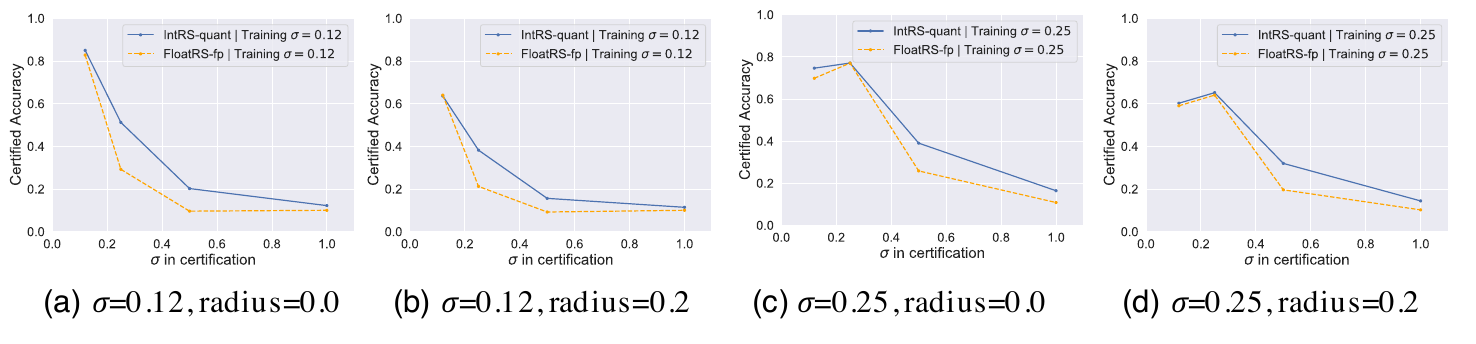}
    \caption{Comparison between IntRS-quant and FloatRS-fp on CIFAR-10. Each model is trained with a fixed $\sigma$ while certifying using different $\sigma$. We perform experiments under train $\sigma=0.12,0.25$ and report certified accuracy for radius $= 0.0, 0.2$.}
    \label{fig:inference}
\end{figure*}
\begin{figure*}
    \centering
    \includegraphics[scale=1.15]{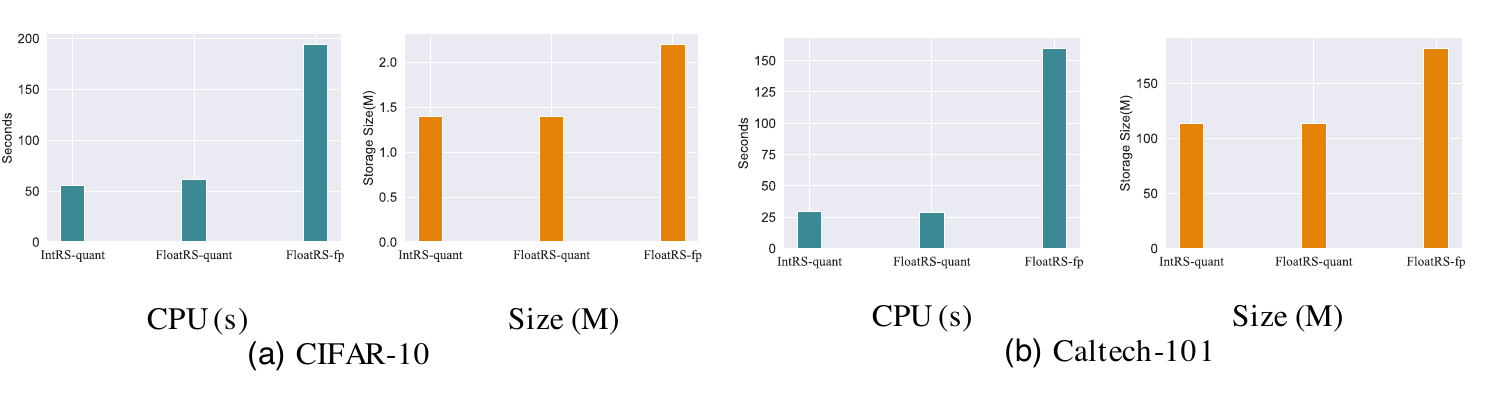}
    \vspace{-1em}
    \caption{Computational efficiency between IntRS-quant and FloatRS-fp on CIFAR-10 and Caltech-101. }
    \label{fig:efficiency_evaluation}
\end{figure*}

\subsection{Experiment Setup}

\partitle{Datasets} We evaluate our framework on two standard image classification datasets: CIFAR-10 \cite{cifar10} and Caltech-101 \cite{fei2004learning}. 

\noindent \textbf{CIFAR-10} consists of 50,000 training images and 10,000 test images, where each image is of $32\times 32$ resolution. 


\noindent \textbf{Caltech-101} contains 9,144 images of size $300 \times 200$ pixels in 102 categories (one of which is background). We use 101 categories for classification (without the background category). We randomly split $80\%$ for training and the remaining images for testing. 


\partitle{Evaluation Metrics}
For evaluation metrics, we use certified percentage and certified accuracy\cite{cohen2019certified}, as follows:

\noindent \textbf{1) Certified Percentage (CP)} is defined as the percentage of clean inputs that satisfy the certification criteria in Theorem  \ref{thm:Certified robustness via Integer Gaus} under the corresponding $\ell_2$-norm ball, which can be formulated as 
\begin{equation}
    \frac{\sum_{i=1}^{L}|\text{Certification}(N_1,\x_i,\sigma)|}{L},
\end{equation}
where $|\text{Certification}(N_1,\x_i,\sigma)|=1$ if it does not abstain and 0, otherwise. $L$ represents the size of dataset.


\noindent \textbf{2) Certified Accuracy (CA)} at radius $r$ is defined as the fraction of the test set in which the randomized smoothing function $g$ makes a correct prediction without abstaining within an $\ell_2$ ball of radius r, which can be formulated as 
\begin{equation}
\frac{\sum_{i=1}^{L} [|\text{Certification}(N_1,\x_i,\sigma)| \text{ }\& \text{ }\hat{c}_i ==c_i \text{ } \&\text{ } R_i\geq r ] }{\sum_{i=1}^{L} | \text{Certification}(N_1,\x_i,\sigma)|},
\end{equation}
where $\hat{c}_i,R_i$ are the returned results for certification for $\x_i$. $c_i$ is $x_i$'s ground-truth label.

\partitle{Comparison Methods}
We compare our proposed IntRS method with \\
\noindent \textbf{1) FloatRS-fp:} Cohen et al. \cite{cohen2019certified} as it was the state-of-the-art provable defence to $\ell_2$-norm bounded adversarial perturbation attacks. We denote it as FloatRS-fp (short for floating-point randomized smoothing for full precision neural networks). \\
\noindent \textbf{2) FloatRS-quant:} We implement a vanilla randomized smoothing approach for quantized neural networks, which first adds continuous Gaussian noises to the images and then quantizes the perturbed images. We denote it as FloatRS-quant (short for floating-point randomized smoothing for quantized neural networks). We stress that FloatRS-quant is incompatible with a real integer-arithmetic-only device since the continuous noise injection step requires floating-point operations. 

To conduct a fair comparison, we run experiments with the same training settings when possible (i.e., use the same quantization-aware and noise-aware strategy for the corresponding neural network) and report the comparison results in metrics of CP and CA.

\noindent \textbf{Implementation Details} 
We implement our algorithm in Pytorch \cite{paszke2019pytorch}.  We consider ResNet-20 \cite{he2016deep} as the base model architecture for CIFAR-10 and ResNet-50 for Caltech-101. In addition, we adopt a model pretrained from ImageNet for Caltech-101 as the images from these two datasets are similar. On each dataset, we train multiple smoothed classifiers with different $\sigma$'s. We compress the model from FP32 representation to int8 representation for the quantized neural networks. We set batch size to 256 for CIFAR-10 and 64 for Caltech-101. The learning rate is set to 0.1 using the step scheduler and SGD optimizer. In all experiments, unless otherwise stated, we set $\alpha=0.001$ for \textbf{Prediction} in Algorithm \ref{alg.certification.prediction}, which corresponds to $99.9\%$ confidence interval, i.e. there is at most a $0.1\%$ chance that \textbf{Certification} falsely certifies a non-robust input. In \textbf{Certification} algorithm,  we use
$N_1 = 100$ and $N_2 =
100,000$ in CIFAR-10. We set $N_2 =
1000$ for Caltech-101 during efficiency comparison. We certify the entire CIFAR-10 testset and 100 images for Caltech-101.

\subsection{Certified Accuracy and Certified Percentage on Adversarially Perturbed Examples}

 Figure \ref{fig:CA_CIFAR} and Figure \ref{fig:CA_Caltech} plot CA obtained by smoothing with different $\sigma$ over a range of $\ell_2$ at radius $r$. The blue, orange, green lines are the CA on adversarial examples for IntRS-quant, Float-fp, and FloatRS-quant, respectively. Note that CA drops to zero beyond a certain maximum point for each $\sigma$ as there is an upper bound to the radius we can certify. This maximum radius is achieved when $\underline{p_a}$ reaches its maximum and all $N_2$ samples are classified as the same class), where $\underline{p_a}$ and $N_2$ are defined in Algorithm \ref{algoritm:certification}. Figure \ref{fig:CR} reports CP for each dataset. It is observed that IntRS-quant achieves comparable CA and CP on both datasets with mild degradation in CA for a few cases compared with the Float-fp model. This is because quantization reduces the bits per weight and leads to errors by these approximations during computations. Moreover, compared with FloatRS-quant, our algorithm can achieve better performance, especially when $\sigma$ is large, which means IntRS is not only capable of being deployed on a wider selection of integer-arithmetics-only devices but also better in prediction accuracy.

\partitle{Trade-off between Accuracy and Robustness}
 As mentioned by Cohen et al. \cite{cohen2019certified} and \cite{wong2018provable}, we also observe from Figure \ref{fig:CA_CIFAR} and Figure \ref{fig:CA_Caltech} that hyperparameter $\sigma$ controls a trade-off between robustness and accuracy. That is, when $\sigma$ is low, small radii can be certified with high accuracy, but large radii cannot be certified at all; when $\sigma$ is high, larger radii can be certified, but smaller radii are certified at lower accuracy. This phenomenon can also be verified in Figure \ref{fig:CR} with a trade-off between CP and $\sigma$ where we observe higher $\sigma$ leads to lower CP. 

\partitle{Effects of Mismatching Noise Magnitudes for Data Augmentation in Training and Smoothing in Inference} 
According to \cite{cohen2019certified}, they train the base classifier with Gaussian noise data augmentation at variance $\sigma^2$ in their experiments. If the base classifier $f$ does not see the same noise during training, it will not necessarily learn to classify $\x$ under Gaussian noise with its ground truth label. Specifically, they report CA using the same $\sigma$ as the standard deviation of Gaussian noise data augmentation for training and robust inference. However, in reality, the attacker may use different $\sigma$'s during inference time. In Figure \ref{fig:inference}, we present CA of models trained with fixed $\sigma$ under different $\sigma=\{0.12,0.25,0.5,1.0\}$ used for certification. Figure \ref{fig:inference} (a) and (b) show the result for training with data augmentation using $\sigma=0.12$ on CIFAR-10 at radius $r=0.0 \text{ and }2.0$, respectively. Figure \ref{fig:inference} (c) and (d) and show the result for training with $\sigma=0.25$. As shown in the figure, certifying using $\sigma$ different from that applied in training can result in lower CA. The classifier achieves the best CP when the training $\sigma$ equals the certification $\sigma$, and the classifier performs worse when it is certifying with smaller $\sigma$, implying that less randomness is introduced in the inference (in Figure \ref{fig:inference} (c) and (d)). We can also observe that our IntRS approach outperforms Float-fp under all different certification $\sigma$'s. This indicates that IntRS is less sensitive to the mismatch level compared with the full precision model and thus provides more flexibility during certified robust inference.

\subsection{Efficiency Evaluation}

\begin{figure}
    \centering
    \includegraphics[scale=1.15]{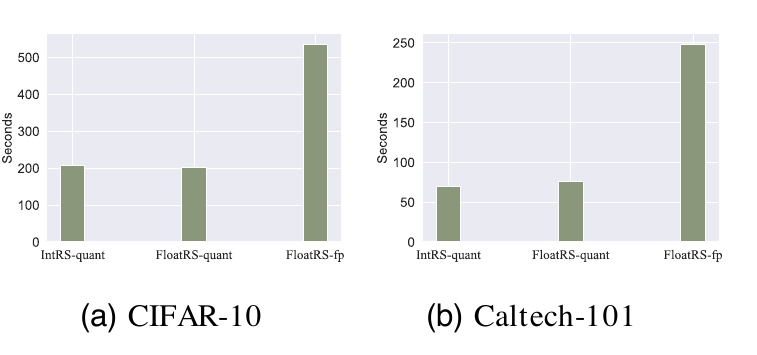}
    \caption{Inference speed between IntRS-quant and FloatRS-fp on mobile device.}
    \label{fig:mobile_efficiency}
\end{figure}

We compare the efficiency performance on two different types of devices: general-purpose CPU and mobile device. For both types, we report the average robust inference time per image on CIFAR-10 and Caltech-101 datasets. The counts of Monte-Carlo repetitions are set to $100,000$ and $10,000$ respectively for CIFAR-10 and Caltech-101.

\partitle{Efficiency Results on General-purpose CPU} 
We test running time on an Intel-i7 CPU-powered desktop computer as an instance of the general-purpose device, which in general has more computational capabilities and power consumption than most edge devices. The results on CIFAR-10 and Caltech-101 are in Figure \ref{fig:efficiency_evaluation}. Since Caltech-101 has fewer Monte-Carlo repetitions, certifying an image on it is faster than that on CIFAR-10. The results show that IntRS-quant for the quantized model requires only 30\% of the original Float-fp for the full precision model in inference time (55.8 s v.s. 194.4 s on CIFAR-10 and 29.5 s vs. 159.9s on Caltech-101 on average). The memory cost can be roughly compared by the storage size of the saved model, which shows ours has approximately $40\%$ less storage than that of a full precision model.

\partitle{Efficiency Results on Mobile Device} Next, we test running time on an iPhone 11 as an instance of the mobile device. The results are reported in Figure \ref{fig:mobile_efficiency}. On the mobile device, it shows clear efficiency improvement of our IntRS for the quantized model over the FP32 model, where the latter costs over 8 mins to robustly predict a single CIFAR-10 image, which makes it impractical to maintain user engagement. Finally, we note that many edge devices are subject to stricter resource and power restrictions, on which the efficiency improvement of randomized smoothing-type techniques made possible by IntRS can be more significant.

\section{Conclusion}

In this paper, we proposed an integer-arithmetic-only randomized smoothing mechanism called IntRS, which has made the certified robustness inference for quantized neural networks ever possible for the first time. We rigorously analyzed its certified robustness property based on the discrete Neyman-Pearson lemma when specified to the discrete Gaussian noise. In practice, we incorporated quantization- and discrete data augmentation- aware training, as well as Monte-Carlo-based practical prediction and certification. We evaluated its effectiveness with modern CNN architectures (ResNet-20 and ResNet-50) on two distinct datasets: CIFAR-10 and Caltech-101. We demonstrated through extensive experiments that IntRS can obtain comparable certified accuracy and certified percentage. More importantly, IntRS makes inference efficient on edge devices. Compared to inferences using models with floating point operations, IntRS requires 40\% times less storage size and gains $4\times$ to $5\times$ acceleration for inference on general-purpose CPUs and mobile devices.

\section{Acknowledgement}
This research has been funded in part by NSF CNS-1952192, NIH CTSA Award UL1TR002378, and Cisco Research Award \#2738379, NSF grants IIS-1910950 and CNS-2027794 and unrestricted cash gifts from Google and Microsoft. Any opinions, findings, and conclusions or recommendations expressed in this material are those of the author(s) and do not necessarily reflect the views of the sponsors.


\newpage
{\small
\bibliographystyle{ieee_fullname}
\bibliography{egbib}

\begin{thebibliography}{10}\itemsep=-1pt

\bibitem{agarwal2019improving}
Chirag Agarwal, Anh Nguyen, and Dan Schonfeld.
\newblock Improving robustness to adversarial examples by encouraging
  discriminative features.
\newblock In {\em 2019 IEEE International Conference on Image Processing
  (ICIP)}, pages 3801--3505. IEEE, 2019.

\bibitem{athalye2018robustness}
Anish Athalye and Nicholas Carlini.
\newblock On the robustness of the cvpr 2018 white-box adversarial example
  defenses.
\newblock {\em arXiv preprint arXiv:1804.03286}, 2018.

\bibitem{bojchevski2020efficient}
Aleksandar Bojchevski, Johannes Klicpera, and Stephan G{\"u}nnemann.
\newblock Efficient robustness certificates for discrete data: Sparsity-aware
  randomized smoothing for graphs, images and more.
\newblock In {\em International Conference on Machine Learning}, pages
  1003--1013. PMLR, 2020.

\bibitem{canonne2020discrete}
Cl{\'e}ment Canonne, Gautam Kamath, and Thomas Steinke.
\newblock The discrete gaussian for differential privacy.
\newblock {\em arXiv preprint arXiv:2004.00010}, 2020.

\bibitem{choudhary2020comprehensive}
Tejalal Choudhary, Vipul Mishra, Anurag Goswami, and Jagannathan Sarangapani.
\newblock A comprehensive survey on model compression and acceleration.
\newblock {\em Artificial Intelligence Review}, pages 1--43, 2020.

\bibitem{cohen2019certified}
Jeremy Cohen, Elan Rosenfeld, and Zico Kolter.
\newblock Certified adversarial robustness via randomized smoothing.
\newblock In {\em International Conference on Machine Learning}, pages
  1310--1320. PMLR, 2019.

\bibitem{das2018mixed}
Dipankar Das, Naveen Mellempudi, Dheevatsa Mudigere, Dhiraj Kalamkar, Sasikanth
  Avancha, Kunal Banerjee, Srinivas Sridharan, Karthik Vaidyanathan, Bharat
  Kaul, Evangelos Georganas, et~al.
\newblock Mixed precision training of convolutional neural networks using
  integer operations.
\newblock {\em arXiv preprint arXiv:1802.00930}, 2018.

\bibitem{dong2019hawq}
Zhen Dong, Zhewei Yao, Amir Gholami, Michael~W Mahoney, and Kurt Keutzer.
\newblock Hawq: Hessian aware quantization of neural networks with
  mixed-precision.
\newblock In {\em Proceedings of the IEEE/CVF International Conference on
  Computer Vision}, pages 293--302, 2019.

\bibitem{dvijotham2020framework}
Krishnamurthy~Dj Dvijotham, Jamie Hayes, Borja Balle, Zico Kolter, Chongli Qin,
  Andr{\'a}s Gy{\"o}rgy, Kai Xiao, Sven Gowal, and Pushmeet Kohli.
\newblock A framework for robustness certification of smoothed classifiers
  using f-divergences.
\newblock In {\em ICLR}, 2020.

\bibitem{faraone2019addnet}
Julian Faraone, Martin Kumm, Martin Hardieck, Peter Zipf, Xueyuan Liu, David
  Boland, and Philip~HW Leong.
\newblock Addnet: Deep neural networks using fpga-optimized multipliers.
\newblock {\em IEEE Transactions on Very Large Scale Integration (VLSI)
  Systems}, 28(1):115--128, 2019.

\bibitem{fei2004learning}
Li Fei-Fei, Rob Fergus, and Pietro Perona.
\newblock Learning generative visual models from few training examples: An
  incremental bayesian approach tested on 101 object categories.
\newblock In {\em 2004 conference on computer vision and pattern recognition
  workshop}, pages 178--178. IEEE, 2004.

\bibitem{fischetti2017deep}
Matteo Fischetti and Jason Jo.
\newblock Deep neural networks as 0-1 mixed integer linear programs: A
  feasibility study.
\newblock {\em arXiv preprint arXiv:1712.06174}, 2017.

\bibitem{goodfellow2014explaining}
Ian~J Goodfellow, Jonathon Shlens, and Christian Szegedy.
\newblock Explaining and harnessing adversarial examples.
\newblock {\em arXiv preprint arXiv:1412.6572}, 2014.

\bibitem{gui2019model}
Shupeng Gui, Haotao Wang, Chen Yu, Haichuan Yang, Zhangyang Wang, and Ji Liu.
\newblock Model compression with adversarial robustness: A unified optimization
  framework.
\newblock {\em arXiv preprint arXiv:1902.03538}, 2019.

\bibitem{guo2016dynamic}
Yiwen Guo, Anbang Yao, and Yurong Chen.
\newblock Dynamic network surgery for efficient dnns.
\newblock {\em In Advances in neural information processing systems}, 2016.

\bibitem{gupta2015deep}
Suyog Gupta, Ankur Agrawal, Kailash Gopalakrishnan, and Pritish Narayanan.
\newblock Deep learning with limited numerical precision.
\newblock In {\em International conference on machine learning}, pages
  1737--1746. PMLR, 2015.

\bibitem{han2015deep}
Song Han, Huizi Mao, and William~J Dally.
\newblock Deep compression: Compressing deep neural networks with pruning,
  trained quantization and huffman coding.
\newblock {\em arXiv preprint arXiv:1510.00149}, 2015.

\bibitem{he2016deep}
Kaiming He, Xiangyu Zhang, Shaoqing Ren, and Jian Sun.
\newblock Deep residual learning for image recognition.
\newblock In {\em CVPR}, pages 770--778, 2016.

\bibitem{hubara2016binarized}
Itay Hubara, Matthieu Courbariaux, Daniel Soudry, Ran El-Yaniv, and Yoshua
  Bengio.
\newblock Binarized neural networks.
\newblock In {\em Proceedings of the 30th international conference on neural
  information processing systems}, pages 4114--4122. Citeseer, 2016.

\bibitem{jacob2018quantization}
Benoit Jacob, Skirmantas Kligys, Bo Chen, Menglong Zhu, Matthew Tang, Andrew
  Howard, Hartwig Adam, and Dmitry Kalenichenko.
\newblock Quantization and training of neural networks for efficient
  integer-arithmetic-only inference.
\newblock In {\em Proceedings of the IEEE Conference on Computer Vision and
  Pattern Recognition}, pages 2704--2713, 2018.

\bibitem{katz2017reluplex}
Guy Katz, Clark Barrett, David~L Dill, Kyle Julian, and Mykel~J Kochenderfer.
\newblock Reluplex: An efficient smt solver for verifying deep neural networks.
\newblock In {\em International Conference on Computer Aided Verification},
  pages 97--117. Springer, 2017.

\bibitem{kurakin2016adversarial}
Alexey Kurakin, Ian Goodfellow, and Samy Bengio.
\newblock Adversarial machine learning at scale.
\newblock {\em arXiv preprint arXiv:1611.01236}, 2016.

\bibitem{lecuyer2019certified}
Mathias Lecuyer, Vaggelis Atlidakis, Roxana Geambasu, Daniel Hsu, and Suman
  Jana.
\newblock Certified robustness to adversarial examples with differential
  privacy.
\newblock In {\em 2019 IEEE Symposium on Security and Privacy (SP)}, pages
  656--672. IEEE, 2019.

\bibitem{DBLP:conf/nips/LeeYCJ19}
Guang{-}He Lee, Yang Yuan, Shiyu Chang, and Tommi~S. Jaakkola.
\newblock Tight certificates of adversarial robustness for randomly smoothed
  classifiers.
\newblock In Hanna~M. Wallach, Hugo Larochelle, Alina Beygelzimer, Florence
  d'Alch{\'{e}}{-}Buc, Emily~B. Fox, and Roman Garnett, editors, {\em Advances
  in Neural Information Processing Systems 32: Annual Conference on Neural
  Information Processing Systems 2019, NeurIPS 2019, December 8-14, 2019,
  Vancouver, BC, Canada}, pages 4911--4922, 2019.

\bibitem{li2018certified}
Bai Li, Changyou Chen, Wenlin Wang, and Lawrence Carin.
\newblock Certified adversarial robustness with additive noise.
\newblock {\em arXiv preprint arXiv:1809.03113}, 2018.

\bibitem{li2018robust}
Yuezun Li, Daniel Tian, Ming-Ching Chang, Xiao Bian, and Siwei Lyu.
\newblock Robust adversarial perturbation on deep proposal-based models.
\newblock {\em arXiv preprint arXiv:1809.05962}, 2018.

\bibitem{liu2018unsupervised}
Qun Liu and Supratik Mukhopadhyay.
\newblock Unsupervised learning using pretrained cnn and associative memory
  bank.
\newblock In {\em 2018 International Joint Conference on Neural Networks
  (IJCNN)}, pages 01--08. IEEE, 2018.

\bibitem{luo2018towards}
Bo Luo, Yannan Liu, Lingxiao Wei, and Qiang Xu.
\newblock Towards imperceptible and robust adversarial example attacks against
  neural networks.
\newblock In {\em Proceedings of the AAAI Conference on Artificial
  Intelligence}, volume~32, 2018.

\bibitem{madry2017towards}
Aleksander Madry, Aleksandar Makelov, Ludwig Schmidt, Dimitris Tsipras, and
  Adrian Vladu.
\newblock Towards deep learning models resistant to adversarial attacks.
\newblock {\em arXiv preprint arXiv:1706.06083}, 2017.

\bibitem{neyman1933ix}
Jerzy Neyman and Egon~Sharpe Pearson.
\newblock Ix. on the problem of the most efficient tests of statistical
  hypotheses.
\newblock {\em Philosophical Transactions of the Royal Society of London.
  Series A, Containing Papers of a Mathematical or Physical Character},
  231(694-706):289--337, 1933.

\bibitem{cifar10}
University of Toronto.
\newblock Learning multiple layers of features from tiny images.
\newblock 2012.

\bibitem{papernot2016limitations}
Nicolas Papernot, Patrick McDaniel, Somesh Jha, Matt Fredrikson, Z~Berkay
  Celik, and Ananthram Swami.
\newblock The limitations of deep learning in adversarial settings.
\newblock In {\em 2016 IEEE European symposium on security and privacy
  (EuroS\&P)}, pages 372--387. IEEE, 2016.

\bibitem{paszke2019pytorch}
Adam Paszke, Sam Gross, Francisco Massa, Adam Lerer, James Bradbury, Gregory
  Chanan, Trevor Killeen, Zeming Lin, Natalia Gimelshein, Luca Antiga, et~al.
\newblock Pytorch: An imperative style, high-performance deep learning library.
\newblock In {\em NIPS}, pages 8026--8037, 2019.

\bibitem{sehwag2019towards}
Vikash Sehwag, Shiqi Wang, Prateek Mittal, and Suman Jana.
\newblock Towards compact and robust deep neural networks.
\newblock {\em arXiv preprint arXiv:1906.06110}, 2019.

\bibitem{sehwag2020hydra}
Vikash Sehwag, Shiqi Wang, Prateek Mittal, and Suman Jana.
\newblock Hydra: Pruning adversarially robust neural networks.
\newblock {\em Advances in Neural Information Processing Systems (NeurIPS)}, 7,
  2020.

\bibitem{singh2019boosting}
Gagandeep Singh, Timon Gehr, Markus P{\"u}schel, and Martin~T Vechev.
\newblock Boosting robustness certification of neural networks.
\newblock In {\em ICLR (Poster)}, 2019.

\bibitem{szegedy2013intriguing}
Christian Szegedy, Wojciech Zaremba, Ilya Sutskever, Joan Bruna, Dumitru Erhan,
  Ian Goodfellow, and Rob Fergus.
\newblock Intriguing properties of neural networks.
\newblock {\em arXiv preprint arXiv:1312.6199}, 2013.

\bibitem{tjeng2017evaluating}
Vincent Tjeng, Kai Xiao, and Russ Tedrake.
\newblock Evaluating robustness of neural networks with mixed integer
  programming.
\newblock {\em arXiv preprint arXiv:1711.07356}, 2017.

\bibitem{wang2020towards}
Peisong Wang, Qiang Chen, Xiangyu He, and Jian Cheng.
\newblock Towards accurate post-training network quantization via bit-split and
  stitching.
\newblock In {\em International Conference on Machine Learning}, pages
  9847--9856. PMLR, 2020.

\bibitem{wang2020apq}
Tianzhe Wang, Kuan Wang, Han Cai, Ji Lin, Zhijian Liu, Hanrui Wang, Yujun Lin,
  and Song Han.
\newblock Apq: Joint search for network architecture, pruning and quantization
  policy.
\newblock In {\em Proceedings of the IEEE/CVF Conference on Computer Vision and
  Pattern Recognition}, pages 2078--2087, 2020.

\bibitem{wang2021certified}
Wenjie Wang, Pengfei Tang, Jian Lou, and Li Xiong.
\newblock Certified robustness to word substitution attack with differential
  privacy.
\newblock In {\em Proceedings of the 2021 Conference of the North American
  Chapter of the Association for Computational Linguistics: Human Language
  Technologies}, pages 1102--1112, 2021.

\bibitem{weng2018towards}
Lily Weng, Huan Zhang, Hongge Chen, Zhao Song, Cho-Jui Hsieh, Luca Daniel,
  Duane Boning, and Inderjit Dhillon.
\newblock Towards fast computation of certified robustness for relu networks.
\newblock In {\em International Conference on Machine Learning}, pages
  5276--5285. PMLR, 2018.

\bibitem{wong2018provable}
Eric Wong and Zico Kolter.
\newblock Provable defenses against adversarial examples via the convex outer
  adversarial polytope.
\newblock In {\em International Conference on Machine Learning}, pages
  5286--5295. PMLR, 2018.

\bibitem{wu2018mixed}
Bichen Wu, Yanghan Wang, Peizhao Zhang, Yuandong Tian, Peter Vajda, and Kurt
  Keutzer.
\newblock Mixed precision quantization of convnets via differentiable neural
  architecture search.
\newblock {\em arXiv preprint arXiv:1812.00090}, 2018.

\bibitem{ye2020safer}
Mao Ye, Chengyue Gong, and Qiang Liu.
\newblock Safer: A structure-free approach for certified robustness to
  adversarial word substitutions.
\newblock In {\em Proceedings of the 58th Annual Meeting of the Association for
  Computational Linguistics}, pages 3465--3475, 2020.

\bibitem{ye2019adversarial}
Shaokai Ye, Kaidi Xu, Sijia Liu, Hao Cheng, Jan-Henrik Lambrechts, Huan Zhang,
  Aojun Zhou, Kaisheng Ma, Yanzhi Wang, and Xue Lin.
\newblock Adversarial robustness vs. model compression, or both?
\newblock In {\em Proceedings of the IEEE/CVF International Conference on
  Computer Vision}, pages 111--120, 2019.

\bibitem{zhang2018lq}
Dongqing Zhang, Jiaolong Yang, Dongqiangzi Ye, and Gang Hua.
\newblock Lq-nets: Learned quantization for highly accurate and compact deep
  neural networks.
\newblock In {\em Proceedings of the European conference on computer vision
  (ECCV)}, pages 365--382, 2018.

\bibitem{zhou2016dorefa}
Shuchang Zhou, Yuxin Wu, Zekun Ni, Xinyu Zhou, He Wen, and Yuheng Zou.
\newblock Dorefa-net: Training low bitwidth convolutional neural networks with
  low bitwidth gradients.
\newblock {\em arXiv preprint arXiv:1606.06160}, 2016.

\bibitem{zhou2017balanced}
Shu-Chang Zhou, Yu-Zhi Wang, He Wen, Qin-Yao He, and Yu-Heng Zou.
\newblock Balanced quantization: An effective and efficient approach to
  quantized neural networks.
\newblock {\em Journal of Computer Science and Technology}, 32(4):667--682,
  2017.

\end{thebibliography}
}

\onecolumn

\begin{appendices}

\begin{center}
\Large{\bf{Integer-arithmetic-only Certified Robustness for Quantized Neural Networks: Supplementary Material}}\vspace*{24pt}\\
\vspace*{12pt}
\end{center}

\section{Additional Details of Training}

\noindent \textbf{Dataset Details} 

\noindent\textbf{CIFAR-10} \cite{cifar10} consists of 50,000 training images and 10,000 test images, where each image is of $32\times 32$ resolution. For data pre-processing, we do horizontal flips and take random crops from images padded by 4 pixels on each side, filling missing pixels with reflections of original images. \\
\textbf{Caltech-101} \cite{fei2004learning} is a more challenging dataset than CIFAR-10 since it contains 9,144 images of size $300 \times 200$ pixels in 102 categories (one of which is background). We use 101 categories for classification (without the background category). We randomly split $80\%$ for training and the remaining images for testing. Following \cite{liu2018unsupervised}, all images are resized and center cropped into $224 \times 224$. We train on the training dataset and test on the testing for both dataset.

\noindent \textbf{Training details} 

\noindent On CIFAR-10, we trained using SGD on one GeForce RTX 2080 GPU. We train for 90 epochs. We use a batch size of 256, and an initial learning rate of 0.1 which drops by a factor of 10 every 30 epochs. On Caltech-101 we trained with SGD on one TITAN RTX GPU. We train for 90 epochs. We use a batch size of 64, and an initial learning rate of 0.1 which drops by a factor of 10 every 30 epochs. The models used in this paper are similar to those used in Cohen et al. \cite{cohen2019certified} except we use a smaller model on CIFAR10. On CIFAR-10, we used a 20-layer residual network from \href{https://github.com/bearpaw/pytorch-classification}{https://github.com/bearpaw/pytorch-classification}. On Caltech-101 our base classifier used the pretrained ResNet-50 architecture provided in torchvision. 

\section{Additional Details and Results of Figure \ref{fig:attack}: The Demonstration Example}

\noindent For Figure \ref{fig:attack} in the paper, we use a well-studied adversarial perturbation attack method: projected gradient descent (PGD) to find adversarial examples against the base classifier $f$ and assess the performance of the attack on full-precision model and quantized model. We set iterations equal to 7 and vary $\varepsilon$ which is the maximum allowed $l_\infty$ perturbation of the input from $0.001$ to $0.05$. Here we present the adversarial examples we found under different $\varepsilon$ in Figure \ref{fig:pgd_small}. For $\varepsilon =0.05$, the adversarial examples generated by PGD attack is visually indistinguishable from the original image, but completely distorts both the full-precision and quantized classifiers' prediction. 

\begin{figure}[H]
    \centering
    \includegraphics[width=7cm]{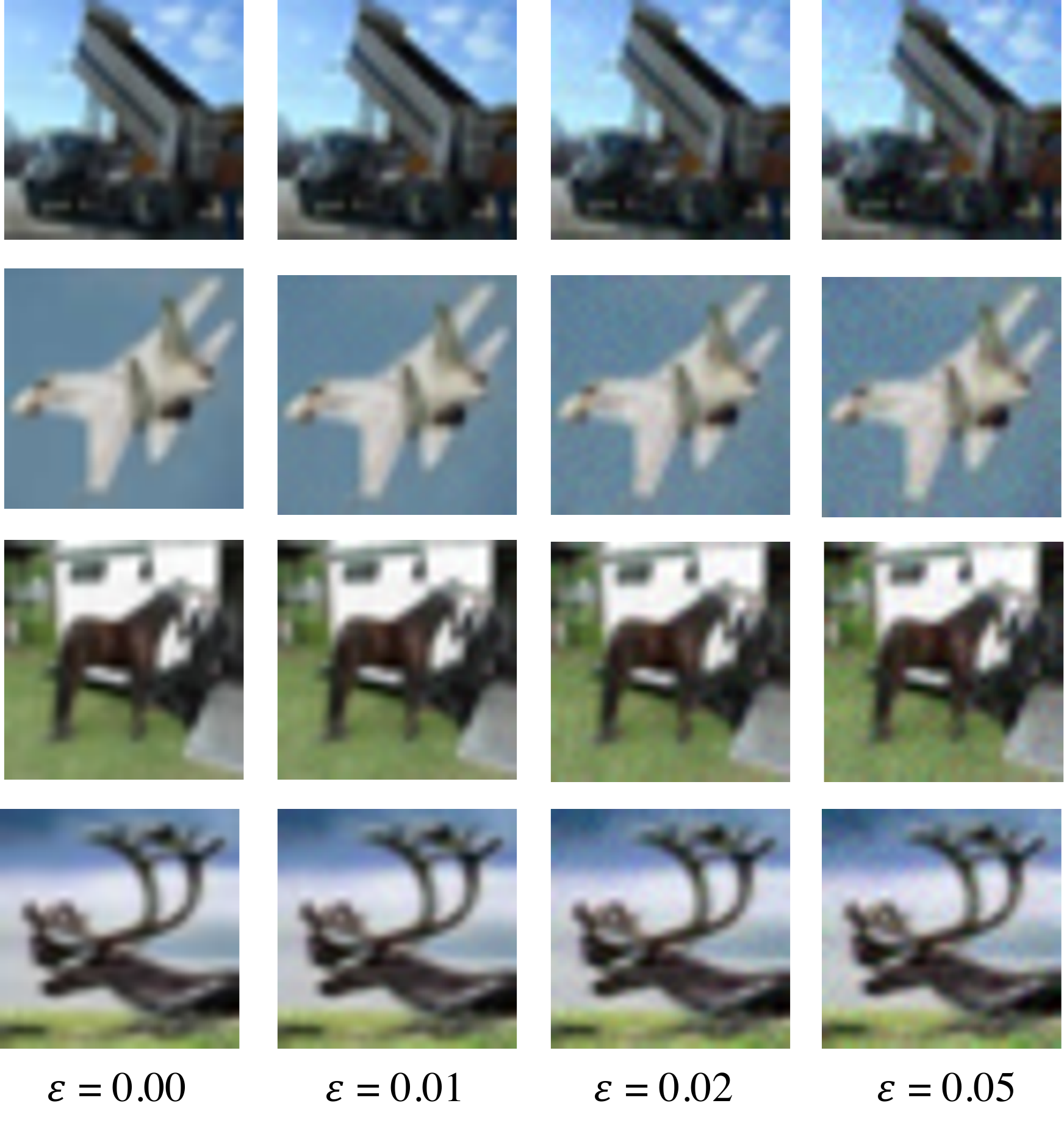}
    \caption{CIFAR-10 adversarial images corrupted generated by PGD attack with varying levels of perturbations}
    \label{fig:pgd_small}
\end{figure}

\section{Practical Prediction and Experiment Results}

\noindent Here we describe how to get the smoothed classifier’s prediction. We use the same prediction algorithm as in \cite{cohen2019certified}. \textbf{Prediction} draws n samples of $f(\x+\n)$ and return the class as its predicted label which appeared much more often than any other class. If such class doesn't exist, \textbf{Prediction} will abstain. The pseudocode is in Algorithm \ref{alg.prediction}. 

\noindent We also analyze the effect of the number of Monte-Carlo samples n in \textbf{Prediction} on quantized model. Table \ref{table:prediction}  shows the performance of \textbf{Prediction} as the number of Monte Carlo samples n is varied between 100 and 10000 on CIFAR-10. When $N$ increases, the time spent on \textbf{Prediction} also increases. We observe from Table \ref{table:prediction} that when $n$ is small, the smooth classifier is more likely to make abstentions for both full-precision (FloatRS-fp) and quantized (IntRS-quant) model.

\begin{algorithm}[h]
\caption{Monte-Carlo estimation and aggregated evaluation for certified robust prediction}
\label{alg.prediction}
\begin{algorithmic}[1]
\REQUIRE Base function $f(\cdot)$, inference sample $\x$, Gaussian noise std $\sigma$, repeated number $N$,  and confidence level $\alpha$.\\
 \textbf{Prediction:}
\STATE Repeat N inferences on $f(\x+\n)$, where $\n\sim \mathcal{N}(0,\sigma^2 \bm{I}_d)$.
\STATE Collect prediction results: $(n_A,\hat{c}_A):$ highest prediction count and its label; $(n_B,\hat{c}_B):$ second highest prediction count and its label;
\IF{Binomial $p$-value test of given $n_A,n_A+n_B$ is no greater than 0.5}
\STATE \textbf{Return} $\hat{c}_A$;
\ELSE \STATE ABSTAIN
\ENDIF

\end{algorithmic}
\label{algoritm:prediction}
\end{algorithm}

\begin{table}[!h]
\caption{Performance of Prediction when $n$ is varied. The column presents the result on CIFAR-10 and set $\sigma=0.25, \alpha=0.001$ The column is "correct'' if Prediction returns the label without abstention and the labels matches with the ground-truth label}
\begin{tabularx}{\textwidth}{|X|X|X|X|X|X|}
\hline
\multicolumn{3}{|c|}{FloatRS-fp} & \multicolumn{3}{|c|}{IntRS-quant}  \\
\hline
n &  correct &  abstain & n &  correct &  abstain \\
\hline
100    &     0.74 &     0.16  &   100    &   0.73 &     0.15 \\
1000   &     0.79 &     0.03  &   1000   &   0.77 &     0.05 \\
10000  &     0.81 &     0.02  &   10000  &   0.80 &     0.02 \\
100000 &     0.82 &     0.00  &   100000 &   0.80 &     0.00 \\
\hline
\end{tabularx}

\label{table:prediction}
\end{table}

\newpage
\section{Additional Experiments with Different Types of Adversarial Perturbation Attacks}

\noindent In this appendix, we use one of the strongest attacks (i.e., projected gradient descent (PGD)) under $\ell_2$ abll to generate adversarial perturbations and evaluate \textbf{Prediction} performance. For \textbf{Prediction}, we set $n=1000,\alpha=0.001,\sigma=0.25$. For PGD, we set 20 iterations and vary $\varepsilon= \{0.0,0.12,0.25,0.50,1.00\}$. Here $\varepsilon$ is the maximum allowed $\ell_2$ perturbation of the input. Figure \ref{fig:pgd_prediction_acc} demonstrates the results of prediction accuracy on adversarial examples of CIFAR-10 on full-precision model and our quantized model.

\begin{figure}[ht]
    \centering
    \includegraphics[width=8cm]{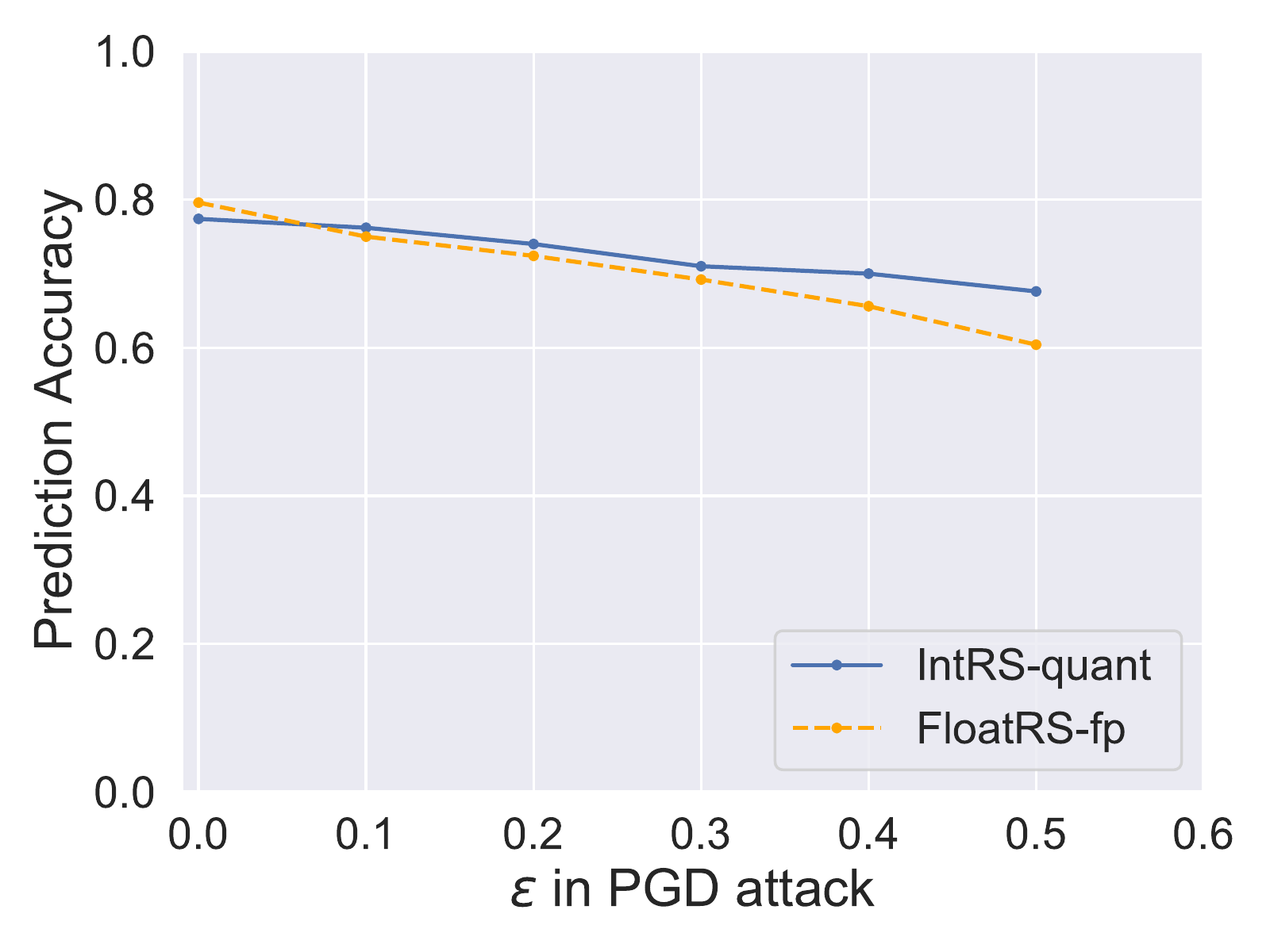}
    \caption{Prediction accuracy on CIFAR-10 adversarial examples of FloatRS-fp and IntRS-quantized model.}
    \label{fig:pgd_prediction_acc}
\end{figure}

\section{Effect of the Confidence Level Parameter $\bm{\alpha}$}

\noindent  In this section, we show the effect of confidence level parameter $\alpha$ on certified accuracy on the full-precision model and our quantized model. We can observe that the certified accuracy of each model has not been vastly affected by choice of $\bm{\alpha}$.

\begin{figure}[ht]
    \centering
    \includegraphics{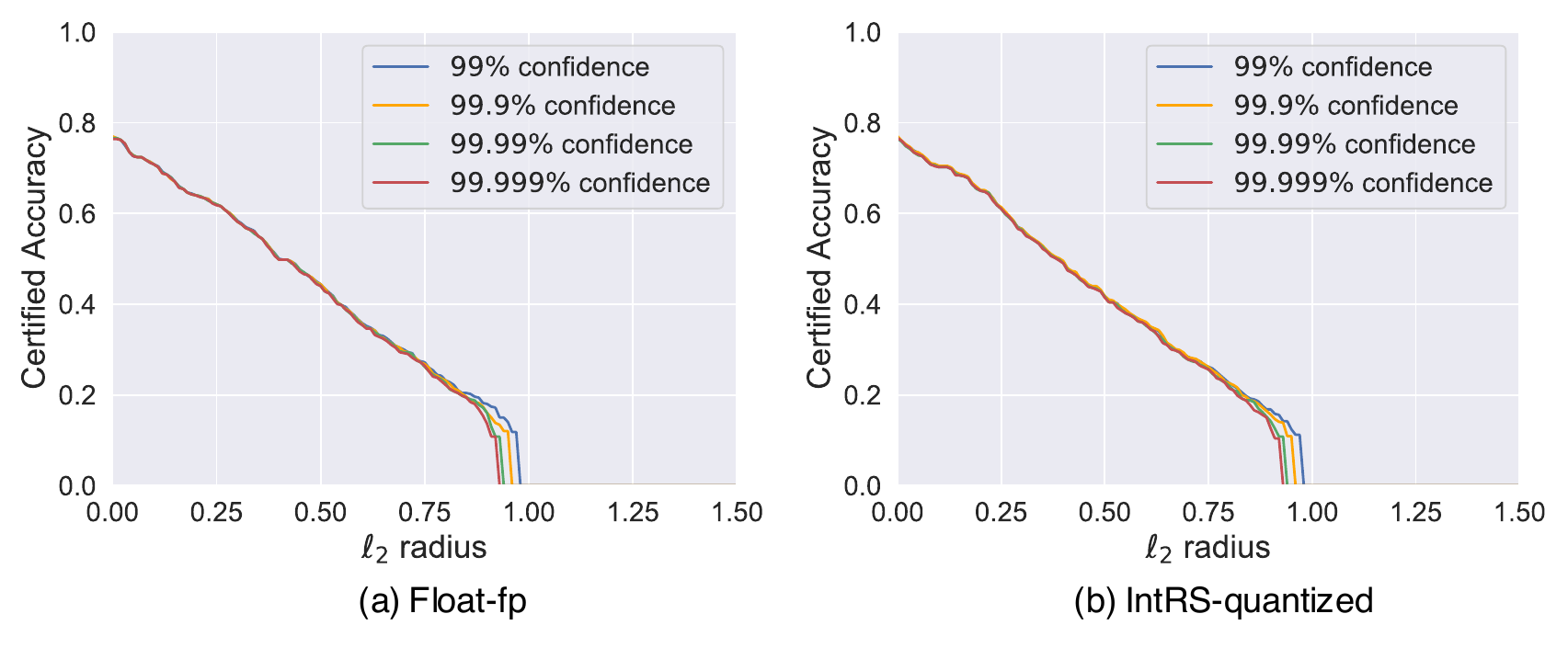}
    \caption{Certified accuracy of varying $bm{\alpha}$. The experiment is performed on CIFAR-10 with $\sigma=0.25$}
    \label{fig:vary_alpha}
\end{figure}

\pagebreak

\section{Detailed Results on Dataset: Report Table}

\noindent In Table \ref{table:cifar10}, \ref{table:caltech101}, we summarize the certified accuracy under different noise level $\sigma$ at different radius $r$. In Table \ref{table:vary_sigma_0.12}, \ref{table:vary_sigma_0.25}, we vary certification noise while holding training noise fixed at $\sigma=0.12,0.25$ on CIFAR-10 to evaluate the effects of Gaussian noise for training base classifier  $f$ on certification performance. Note for the quantized model, the accuracy of base model $f$ would be slightly lower than that of the full-precision model. Our goal is to achieve comparably certified accuracy for IntRS-quant compared with FloatRS-fp model.

\begin{table}[ht]

\caption{Certified test accuracy on CIFAR-10 with different $\sigma$. Each column represents the certified accuracy at different radius $r$}
\begin{tabularx}{\textwidth}{X|X X X X X X}
\hline
 FloatRS-fp & $r = 0.25$& $r = 0.5$& $r = 0.75$& $r = 1.0$& $r = 1.25$& $r = 1.5$\\
 \hline
  $\sigma = 0.12$ & 0.59 & 0.00 & 0.00 & 0.00 & 0.00 & 0.00\\
  $\sigma = 0.25$ & 0.62 & 0.44 & 0.27 & 0.00 & 0.00 & 0.00\\
  $\sigma = 0.50$ & 0.54 & 0.43 & 0.32 & 0.22 & 0.15 & 0.09\\
  $\sigma = 1.00$ & 0.39 & 0.33 & 0.28 & 0.22 & 0.18 & 0.15\\
\hline
    IntRS-quant   & $r = 0.25$& $r = 0.5$& $r = 0.75$& $r = 1.0$& $r = 1.25$& $r = 1.5$\\
    \hline
  $\sigma = 0.12$ & 0.59 & 0.00 & 0.00 & 0.00 & 0.00 & 0.00\\
  $\sigma = 0.25$ & 0.61 & 0.42 & 0.26 & 0.00 & 0.00 & 0.00\\
  $\sigma = 0.50$ & 0.52 & 0.39 & 0.29 & 0.22 & 0.15 & 0.08\\
  $\sigma = 1.00$ & 0.35 & 0.28 & 0.23 & 0.18 & 0.16 & 0.12\\
  \hline
  FloatRS-quant & $r = 0.25$& $r = 0.5$& $r = 0.75$& $r = 1.0$& $r = 1.25$& $r = 1.5$\\
  \hline
  $\sigma = 0.12$ & 0.56 & 0.00 & 0.00 & 0.00 & 0.00 & 0.00\\
 $\sigma = 0.25$ & 0.59 & 0.42 & 0.23 & 0.00 & 0.00 & 0.00\\
  $\sigma = 0.50$ & 0.43 & 0.33 & 0.25 & 0.18 & 0.11 & 0.06\\
  $\sigma = 1.00$ & 0.19 & 0.14 & 0.12 & 0.09 & 0.07 & 0.05\\
  \hline
\end{tabularx}
\label{table:cifar10}    
\end{table}

\begin{table}[ht]
    
\caption{Certified test accuracy on Caltech-101 with different $\sigma$ }
\begin{tabularx}{\textwidth}{p{2cm}|X X X X X X X X}

\hline
 FloatRS-fp & $r = 0.25$& $r = 0.5$& $r = 0.75$& $r = 1.0$& $r = 1.25$& $r = 1.5$& $r = 1.75$& $r = 2.0$\\
\hline
  $\sigma = 0.12$ & 0.65 & 0.00 & 0.00 & 0.00 & 0.00 & 0.00 & 0.00 & 0.00\\
  $\sigma = 0.25$ & 0.56 & 0.54 & 0.00 & 0.00 & 0.00 & 0.00 & 0.00 & 0.00\\
  $\sigma = 0.50$ & 0.62 & 0.58 & 0.55 & 0.52 & 0.00 & 0.00 & 0.00 & 0.00\\
  $\sigma = 1.00$ & 0.51 & 0.51 & 0.48 & 0.47 & 0.46 & 0.45 & 0.45 & 0.41\\
 \hline
IntRS-quant& $r = 0.25$& $r = 0.5$& $r = 0.75$& $r = 1.0$& $r = 1.25$& $r = 1.5$& $r = 1.75$& $r = 2.0$\\
\hline
  $\sigma = 0.12$ & 0.61 & 0.00 & 0.00 & 0.00 & 0.00 & 0.00 & 0.00 & 0.00\\
  $\sigma = 0.25$ & 0.58 & 0.56 & 0.00 & 0.00 & 0.00 & 0.00 & 0.00 & 0.00\\
  $\sigma = 0.50$ & 0.64 & 0.59 & 0.51 & 0.46 & 0.00 & 0.00 & 0.00 & 0.00\\
  $\sigma = 1.00$ & 0.56 & 0.56 & 0.56 & 0.54 & 0.53 & 0.52 & 0.52 & 0.52\\
  \hline
 FloatRS-quant & $r = 0.25$& $r = 0.5$& $r = 0.75$& $r = 1.0$& $r = 1.25$& $r = 1.5$& $r = 1.75$& $r = 2.0$\\
 \hline
  $\sigma = 0.12$ & 0.59 & 0.00 & 0.00 & 0.00 & 0.00 & 0.00 & 0.00 & 0.00\\
  $\sigma = 0.25$ & 0.60 & 0.56 & 0.00 & 0.00 & 0.00 & 0.00 & 0.00 & 0.00\\
  $\sigma = 0.50$ & 0.61 & 0.56 & 0.00 & 0.00 & 0.00 & 0.00 & 0.00 & 0.00\\
  $\sigma = 1.00$ & 0.20 & 0.18 & 0.18 & 0.18 & 0.16 & 0.14 & 0.12 & 0.12\\
\hline
\end{tabularx}
\label{table:caltech101}
\end{table}

\begin{table}[ht]
 \caption{Certified Accuracy of varying $\sigma$ used in certification. The base model $f$ is trained on CIFAR-10 using Gaussian noise augmentation with $\sigma=0.12$}
\begin{tabularx}{\textwidth}{X|X X X X X X X}
\hline
 FloatRS-fp & $r = 0.25$& $r = 0.5$& $r = 0.75$& $r = 1.0$& $r = 1.25$& $r = 1.5$& $r = 1.75$\\
 \hline
  $\sigma = 0.12$ & 0.59 & 0.00 & 0.00 & 0.00 & 0.00 & 0.00 & 0.00\\
  $\sigma = 0.25$ & 0.19 & 0.11 & 0.07 & 0.00 & 0.00 & 0.00 & 0.00\\
  $\sigma = 0.50$ & 0.09 & 0.09 & 0.08 & 0.07 & 0.04 & 0.01 & 0.00\\
  $\sigma = 1.00$ & 0.10 & 0.09 & 0.09 & 0.08 & 0.06 & 0.04 & 0.03\\
  \hline
IntRS-quant & $r = 0.25$& $r = 0.5$& $r = 0.75$& $r = 1.0$& $r = 1.25$& $r = 1.5$& $r = 1.75$\\
 \hline
  $\sigma = 0.12$ & 0.59 & 0.00 & 0.00 & 0.00 & 0.00 & 0.00 & 0.00\\
  $\sigma = 0.25$ & 0.36 & 0.24 & 0.11 & 0.00 & 0.00 & 0.00 & 0.00\\
  $\sigma = 0.50$ & 0.15 & 0.12 & 0.11 & 0.08 & 0.05 & 0.01 & 0.00\\
  $\sigma = 1.00$ & 0.11 & 0.10 & 0.10 & 0.09 & 0.09 & 0.09 & 0.07\\
   \hline
\end{tabularx}
\label{table:vary_sigma_0.12}
\end{table}

\begin{table}[ht]

 \caption{Certified Accuracy of varying $\sigma$ used in certification. The base model $f$ is trained trained on CIFAR-10 using Gaussian noise augmentation with $\sigma=0.25$}
\begin{tabularx}{\textwidth}{X|X X X X X X X}
\hline
 FloatRS-fp& $r = 0.25$& $r = 0.5$& $r = 0.75$& $r = 1.0$& $r = 1.25$& $r = 1.5$& $r = 1.75$\\
\hline
  $\sigma = 0.12$ & 0.57 & 0.00 & 0.00 & 0.00 & 0.00 & 0.00 & 0.00\\
  $\sigma = 0.25$ & 0.62 & 0.44 & 0.27 & 0.00 & 0.00 & 0.00 & 0.00\\
  $\sigma = 0.50$ & 0.19 & 0.15 & 0.10 & 0.05 & 0.02 & 0.01 & 0.00\\
  $\sigma = 1.00$ & 0.10 & 0.09 & 0.08 & 0.08 & 0.06 & 0.04 & 0.02\\
  \hline
IntRS-quant & $r = 0.25$& $r = 0.5$& $r = 0.75$& $r = 1.0$& $r = 1.25$& $r = 1.5$& $r = 1.75$\\
\hline
  $\sigma = 0.12$ & 0.57 & 0.00 & 0.00 & 0.00 & 0.00 & 0.00 & 0.00\\
  $\sigma = 0.25$ & 0.61 & 0.42 & 0.26 & 0.00 & 0.00 & 0.00 & 0.00\\
  $\sigma = 0.50$ & 0.31 & 0.23 & 0.15 & 0.08 & 0.02 & 0.01 & 0.00\\
  $\sigma = 1.00$ & 0.14 & 0.11 & 0.10 & 0.07 & 0.03 & 0.02 & 0.01\\
   \hline
\end{tabularx}
\label{table:vary_sigma_0.25}
\end{table}

\newpage
\section{Examples of Noisy Images}

\noindent In this section, we demonstrate examples of CIFAR-10 and Caltech-101 images corrupted with varying levels of noise in Gaussian noise. Since it is hard to visualize the quantized input, we only present the input corrupted by $\mathcal{N}(0,\sigma^2)$. 

\begin{figure}[ht]
    \centering
    \includegraphics[width=10cm]{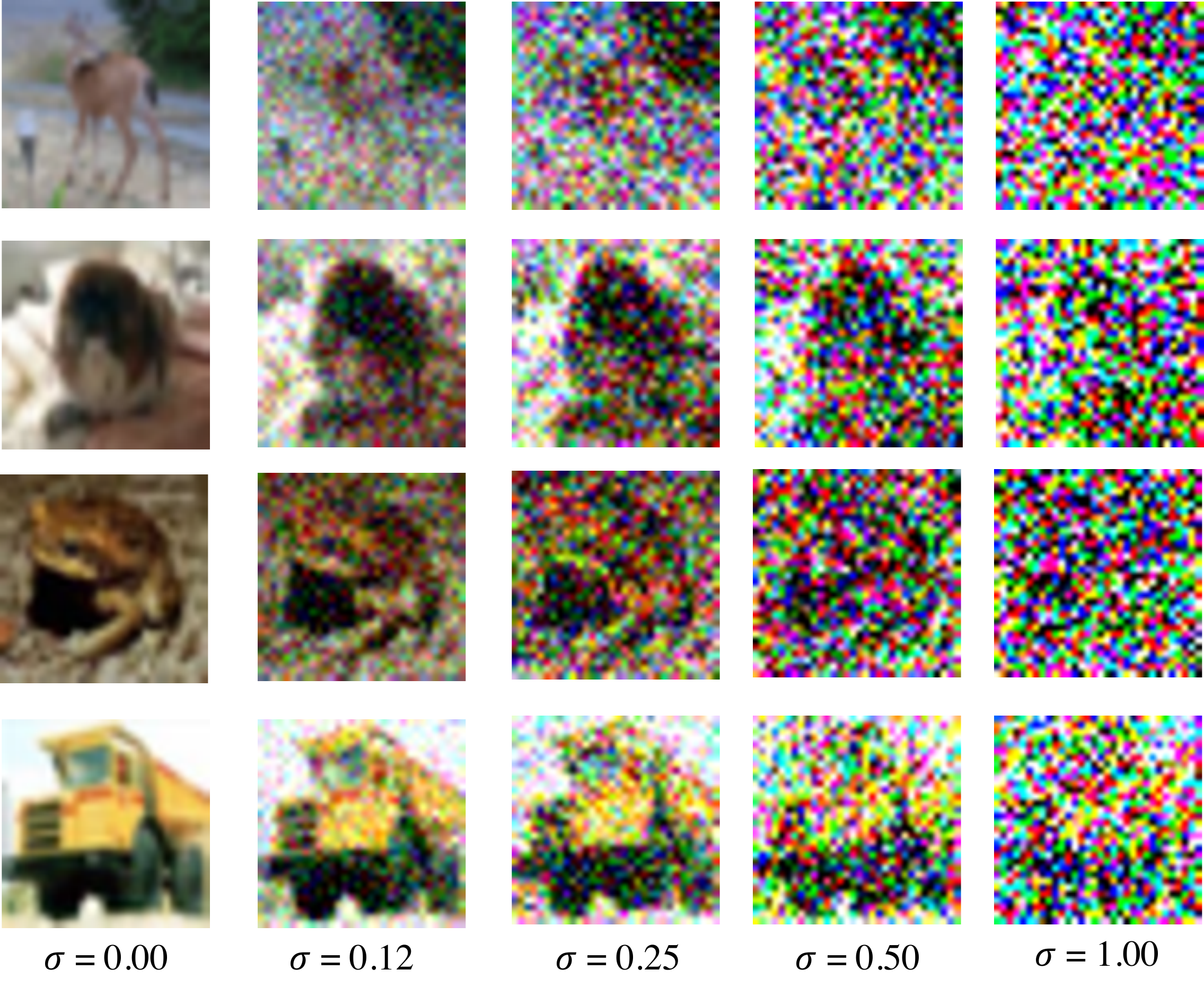}
    \caption{An illustration of CIFAR-10 images generated by adding Gaussian noise with various $\sigma$Pixel values greater than 1.0 or less than 0.0 were clipped to 1.0 or 0.0.}
    \label{fig:cifar_sd_all}
\end{figure}

\begin{figure}[ht]
    \centering
    \includegraphics[width=15cm]{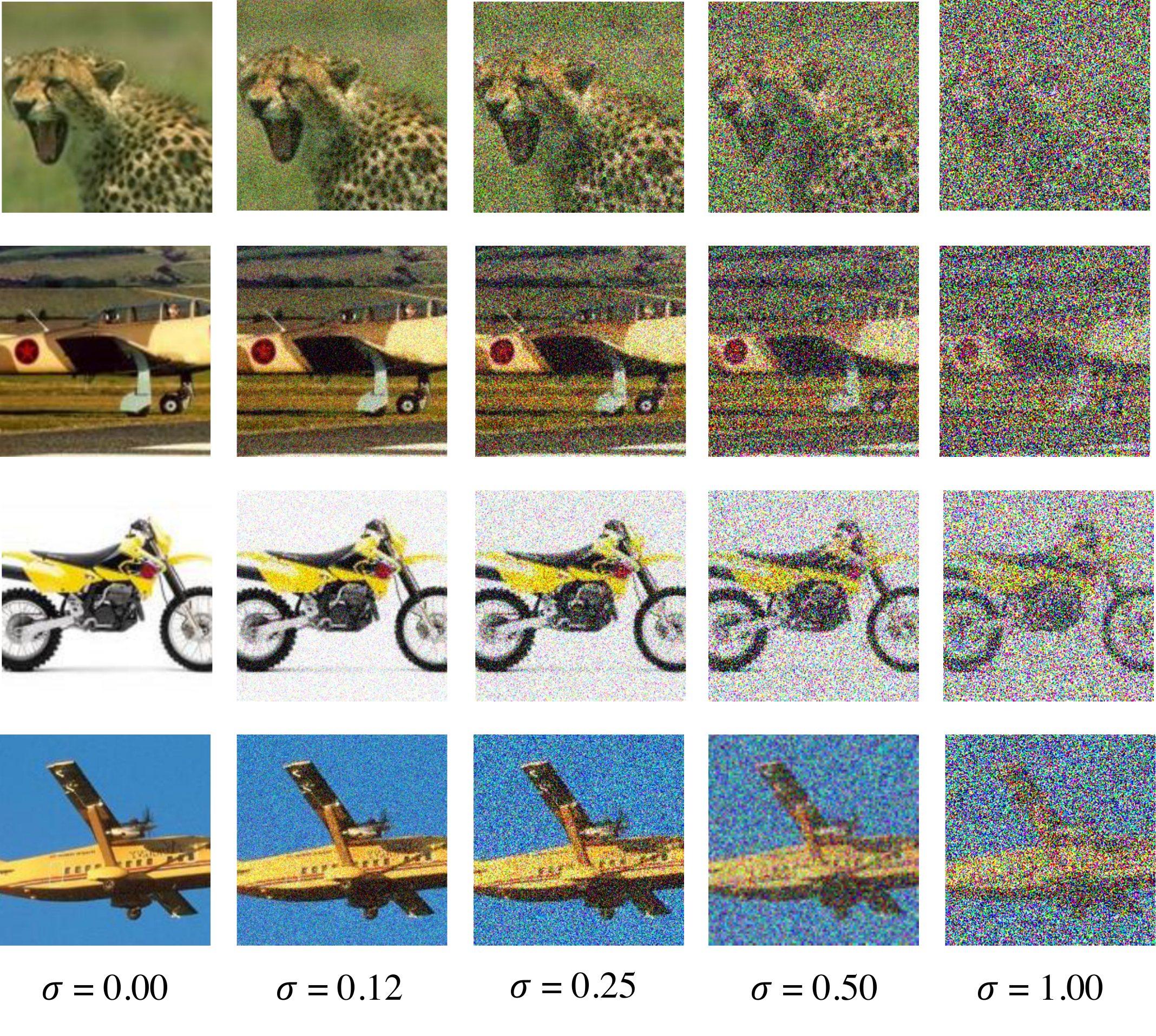}
    \caption{An illustration of Caltech-101 images generated by adding Gaussian noise with various $\sigma$. Pixel values greater than 1.0 or less than 0.0 were clipped to 1.0 or 0.0.}
    \label{fig:caltech_sd_all}
\end{figure}

\clearpage

\section{Omitted Proof}

\subsection{Proof of Proposition 3.1}

\begin{proof}
        The proof follows the Neyman-Pearson lemma (especially its form for discrete distribution in Lemma \ref{lemma.discrete.Neyman.Pearson}). We want to show that the condition in this Proposition is equivalent to the condition with respect to the likelihood ratio statistic, which takes the form:
        \begin{equation}
            \begin{split}
            \tL(\z) = \frac{\Big{(}\prod_{i=1}^d \frac{e^{-(\z_i-(\x_i+\bm{\delta} _i))^2/2\sigma ^2}}{\sum _{\h_i\in \gH}e^{-(\h_i-(\x_i+\bm{\delta} _i))^2/2\sigma ^2}}\Big{)} } { \Big{(}\prod_{i=1}^d \frac{e^{-(\z_i-\x_i)^2/2\sigma ^2}}{\sum _{\h_i\in \gH}e^{-(\h_i-\x_i)^2/2\sigma ^2}}\Big{)} }.
            \end{split} 
        \end{equation}
     Due to the discrete nature of the inference stage, we have $\bm{\delta} _i\in \gH$. Based on this fact and by $\sum _{\h_i\in \gH}e^{-(\h_i-(\x_i+\bm{\delta}_i))^2/2\sigma ^2}$ is periodic for $\bm{\delta} _i\in \gH$, we have
    \begin{equation}
        \sum _{\h_i\in \gH}e^{-(\h_i-(\x_i+\bm{\delta} _i))^2/2\sigma ^2} = \sum _{\h_i\in \gH}e^{-(\h_i-\x_i)^2/2\sigma ^2}.
    \end{equation}
    Then, we further have
        \begin{equation}
            \begin{split}
            \tL(\z) & = \frac{e^{-\sum_{i=1}^d(\z_i-(\x_i+\bm{\delta} _i))^2/2\sigma ^2}}{e^{-\sum_{i=1}^d(\z_i-\x_i)^2/2\sigma ^2}} \\
            & = e^{\frac{1}{2\sigma ^2}\sum_{i=1}^d (2\z_i\bm{\delta} _i-\delta_i-2\x_i\bm{\delta} _i)} \\
            & = e^{\frac{1}{\sigma ^2}\la \z,\bm{\delta} \ra - \frac{1}{2\sigma ^2} (\|\bm{\delta}\|_2^2+2\la \x,\bm{\delta} \ra))}.
            \end{split} 
        \end{equation}

        Thus, in order to carry out the likelihood ratio test, we have the following equivalent relationship.
        \begin{gather}
            \tL(\z) \leq \alpha \Longleftrightarrow \la \z,\bm{\delta} \ra \leq \sigma^2\ln \alpha + \frac{1}{2}(\|\bm{\delta}\|_2^2+2\la \x,\bm{\delta} \ra) \\
            \tL(\z) \geq \alpha \Longleftrightarrow \la \z,\bm{\delta} \ra \geq \sigma^2\ln \alpha + \frac{1}{2}(\|\bm{\delta}\|_2^2+2\la \x,\bm{\delta} \ra).
        \end{gather}
        The remaining follows Lemma \ref{lemma.discrete.Neyman.Pearson}.
\end{proof} 


\end{appendices}

\end{document}